\documentclass[letterpaper]{article} 
\usepackage{aaai25}  
\usepackage{times}  
\usepackage{helvet}  
\usepackage{courier}  
\usepackage[hyphens]{url}  
\usepackage{graphicx} 
\usepackage{natbib}  
\usepackage{caption} 
\usepackage[utf8]{inputenc} 
\usepackage[T1]{fontenc}    
\usepackage{url}            
\usepackage{booktabs}       
\usepackage{amsfonts}       
\usepackage{nicefrac}       
\usepackage{microtype}      
\usepackage{lipsum}
\usepackage{graphicx}
\usepackage[english]{babel}

\usepackage{graphicx}
\usepackage{svg}
\usepackage{multirow}
\usepackage{amsmath,amssymb,amsfonts}
\usepackage{amsthm}
\usepackage{mathrsfs}
\usepackage[title]{appendix}
\usepackage{xcolor}
\usepackage{textcomp}
\usepackage{manyfoot}
\usepackage{booktabs}
\usepackage{algorithm}
\usepackage{algorithmicx}
\usepackage{algpseudocode}
\usepackage{listings}
\usepackage{tikz}
\usetikzlibrary{shapes, arrows.meta, positioning}

\usepackage{amsmath}
\usepackage{amssymb}
\usepackage{amsthm}
\usepackage{algorithm}
\usepackage{algpseudocode}
\usepackage{graphicx}
\usepackage{csquotes}
\usepackage{mathtools}
\usepackage{pst-node}
\usepackage{booktabs}
\usepackage{subcaption}
\usepackage[title]{appendix}
\usepackage{listings}
\usepackage{placeins}
\usepackage{bm}
\usepackage{graphicx}
\usepackage{makecell}
\usepackage{arydshln}
\usepackage{pifont}

\theoremstyle{plain}
\newtheorem{theorem}{Theorem}
\newtheorem{proposition}[theorem]{Proposition}

\newtheorem*{proposition*}{Proposition}
\newtheorem*{lemma*}{Lemma}
\theoremstyle{definition}
\newtheorem{definition}[theorem]{Definition}

\theoremstyle{remark}
\newtheorem{remark}[theorem]{Remark}

\newcommand{\ie}{\textit{i}.\textit{e}., }

\DeclareMathOperator{\diag}{\operatorname{diag}}

\DeclareMathOperator{\Sp}{\operatorname{Sp}}

\algnewcommand\algorithmicInput{\textbf{Input:}}
\algnewcommand\algorithmicOutput{\textbf{Output:}}
\algnewcommand\Input{\item[\algorithmicInput]}
\algnewcommand\Output{\item[\algorithmicOutput]}
\algnewcommand\algorithmicparameters{\textbf{Parameters:}}
\algnewcommand\Parameters{\item[\algorithmicparameters]}

\theoremstyle{definition}

\newcommand{\xmark}{\ding{55}} 
\newcommand{\cmark}{\ding{51}} 

\newcommand{\method}{HYGENE}

\title{HYGENE: A Diffusion-based Hypergraph Generation Method}

\author {
    Dorian Gailhard,
    Enzo Tartaglione,
    Lirida Naviner,
    Jhony H. Giraldo
}
\affiliations {
    LTCI, Télécom Paris, Institut Polytechnique de Paris, France\\
    \{name.surname\}@telecom-paris.fr
}

\begin{document}

\maketitle

\begin{abstract}
Hypergraphs are powerful mathematical structures that can model complex, high-order relationships in various domains, including social networks, bioinformatics, and recommender systems.
However, generating realistic and diverse hypergraphs remains challenging due to their inherent complexity and lack of effective generative models.
In this paper, we introduce a diffusion-based \textbf{Hy}pergraph \textbf{Gene}ration (\method) method that addresses these challenges through a progressive local expansion approach.
\method~works on the bipartite representation of hypergraphs, starting with a single pair of connected nodes and iteratively expanding it to form the target hypergraph.
At each step, nodes and hyperedges are added in a localized manner using a denoising diffusion process, which allows for the construction of the global structure before refining local details.
Our experiments demonstrated the effectiveness of \method, proving its ability to closely mimic a variety of properties in hypergraphs.
To the best of our knowledge, this is the first attempt to employ diffusion models for hypergraph generation. Our code is open source\footnote{\url{https://github.com/DorianGailhard/HYGENE}}.
\end{abstract}

\section{Introduction} \label{sec::introduction}

Hypergraphs are higher-order extensions of graphs.
They comprise a set of nodes, also called vertices, and a set of hyperedges.
Unlike regular graphs, where edges connect only two nodes, hyperedges can connect any number of nodes.
These structures have demonstrated their ability to capture more complex relationships than graphs and have been applied in various domains~\cite{gong2023generativehypergraphmodelsspectral,Estrada_2006}.
For instance, hypergraphs have been applied in drug discovery~\cite{pmlr-v97-kajino19a}, modeling contagion spread~\cite{Higham_2021}, and electronics~\cite{Starostin,luo2024dehnneffectiveneuralmodel,GRZESIAKKOPEC2017491}.
Besides, they have also proven useful in recommender systems~\cite{10336546}, molecular biology~\cite{molecular,molecularbiology}, and urban planning~\cite{urbanplanning}.
The versatility of hypergraphs in representing multi-way relationships makes them a powerful tool across these diverse fields; consequently, hypergraph generation (the ability to sample from specific hypergraph distributions) holds significant promise.

\begin{figure}[t]
    \centering
    \includegraphics[width=\columnwidth]{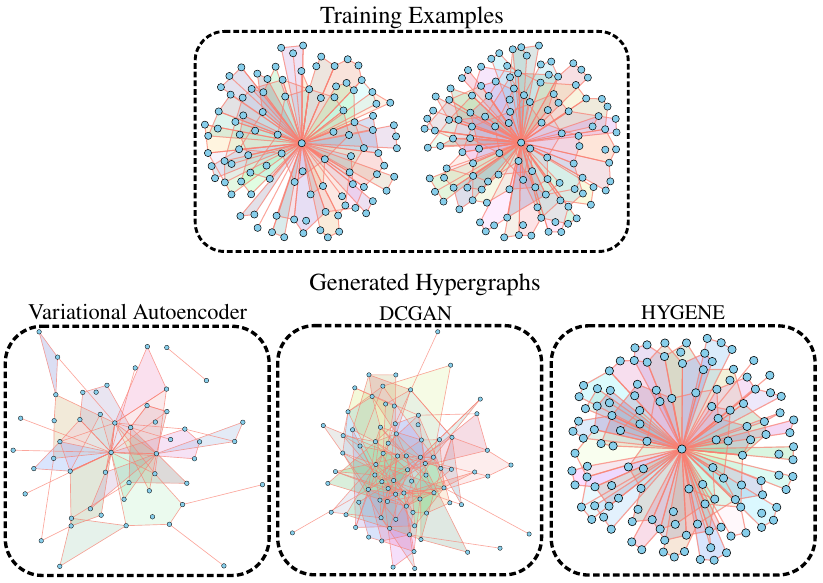}
    \caption{Examples of generated ego hypergraphs by a variational autoencoder, a Deep Convolutional Generative Adversarial Network (DCGAN), and our model (\method).}
    \label{fig:teaser}
\end{figure}

Despite its wide applicability, research in hypergraph generation has primarily focused on algorithmic approaches, aiming to develop methodologies that produce hypergraphs with specific, predefined structural properties~\cite{Do_2020,arafat2020constructionrandomgenerationhypergraphs}.
In contrast, the exploration of hypergraph generation using deep learning models remains largely understudied.
This gap represents a significant opportunity to advance this field since deep learning approaches may capture complex patterns and generate more realistic and diverse hypergraphs.
Such methods could enhance the modeling of intricate relationships beyond the scope of traditional graph structures.


In contrast to hypergraph generation, deep learning-based graph generation has been extensively studied~\cite{zhu2022surveydeepgraphgeneration}.
Learning-based graph generation can be broadly categorized into two approaches: \emph{one-shot} approaches that generate the entire graph simultaneously~\cite{simonovsky2018graphvaegenerationsmallgraphs,you2018graphrnngeneratingrealisticgraphs, vignac2023digress}, and \emph{iterative} models that generate the graphs incrementally, predicting edges for each new node~\cite{chen2023efficientdegreeguidedgraphgeneration, bergmeister2024efficient}.
While graph generation techniques have shown promise, their adaptation to hypergraphs remains challenging.
The variable size of hyperedges and their higher-order relationships increase the difficulty of the task, making direct application of graph methods non-trivial.
Figure~\ref{fig:teaser} shows that naïvely generating the incidence matrix using classical image generation architectures is not sufficient either, as it lacks the correct understanding of the underlying data structure.


By leveraging the spectral equivalence between a hypergraph and two carefully chosen representations (the clique and star expansions), we generalize the work by~\citet{bergmeister2024efficient} for hypergraph generation.
Their method is based on an iterative local expansion scheme, where graph generation is performed hierarchically, first building the global structure before refining the details.
This process is seen as the inverse of a coarsening operation, which involves the reduction of a graph while preserving its relevant properties.
For hypergraphs, the variable size of hyperedges increases the complexity of the problem.
As every non-empty set of nodes is a possible hyperedge, they are exponentially more numerous than classical edges.
In order to mitigate this, we introduce an iterative expansion and refinement process for Hypergraph Generation (\method), rather than predicting all possible hyperedges at once.
We train a denoising diffusion model \cite{karras2022elucidatingdesignspacediffusionbased} using this framework, and validate our method on four synthetic and three real-world datasets, demonstrating its effectiveness in replicating important structural properties.
At a glance, our main contributions are the following:
\begin{itemize}
    \item To the best of our knowledge, we introduce the first diffusion-based method for generating hypergraphs sampled from specific distributions (Sec.~\ref{sec:overview}).
    \item We generalize important concepts in the graph domain to hypergraph generation, like hypergraph coarsening and diffusion (Sec.~\ref{sec:descending}, Sec.~\ref{sec:ascending}, Sec.~\ref{sec:prob}).
    \item We provide rigorous theoretical justifications for our technical choices.
    \item We validate \method~on four synthetic and three real-world datasets, showcasing its ability to capture and reproduce subtle structural properties of hypergraphs (Sec.~\ref{sec::experiments}).
\end{itemize}

\section{Related Work}

\textbf{Graph Generation Using Deep Learning}.
The field of graph generation using deep learning models was pioneered by GraphVAE \cite{simonovsky2018graphvaegenerationsmallgraphs}. 
This approach employs an autoencoder to embed graphs into a latent space, from which new graphs could be generated by sampling and decoding. Subsequently, \citet{you2018graphrnngeneratingrealisticgraphs} significantly enhanced generation quality by using a recurrent neural network to produce the adjacency matrix column-wise.
Recent advancements include the work by \citet{kong2023autoregressive}, which formalized the generation process as an inverse discrete absorption. 
In this method, nodes from a training graph are sequentially absorbed, and a mixture of multinomials is trained to predict the necessary edge additions when reintroducing a node.
Diffusion models, which have shown remarkable success in image generation, were adapted for graph generation in \cite{niu2020permutationinvariantgraphgeneration}.
This approach was further refined by \citet{vignac2023digress} and expanded in \cite{chen2023efficientdegreeguidedgraphgeneration} by incorporating target degree distributions.

A departure from typical methods in graph generation was proposed by \citet{bergmeister2024efficient}.
Instead of sequentially adding nodes and predicting their connections, this approach reverses a \textit{coarsening} process.
During training, graphs are reduced by merging nodes into clusters.
The model then learns to identify these clusters, decompose them, and reconstruct the original connections between their nodes.
Graph generation begins with a coarse single-node graph and progressively adds details by expanding it using the trained model, mirroring the diffusion approaches seen in modern image generation techniques.
Similar hierarchical concepts have been explored in molecule generation, with \citet{zhu2023molhfhierarchicalnormalizingflow} applying normalizing flows.
Related ideas can also be found in the work of \citet{guo2023unpooling}.

In contrast to previous work, we focus on the problem of hypergraph generation, which extends the concept of graph generation to higher-order structures.
Furthermore, we employ a hierarchical view of the problem instead of the sequential edge-by-edge generation commonly employed.
We also depart from the classical view of edge prediction, where a model outputs the probability of existence for all possible edges, and, instead, predict both the number of hyperedges and their composition.

\section{Method}

\subsection{Preliminaries}

\noindent \textbf{Notation}.
In this work, calligraphic letters such as $\mathcal{V}$ denote sets, and $\vert \mathcal{V} \vert$ represents the cardinality of the set.
Uppercase boldface letters such as $\mathbf{A}$ represent matrices, while lowercase boldface letters such as $\mathbf{x}$ denote vectors.
The superscript $(\cdot)^{\mathsf{T}}$ corresponds to transposition. $\diag(\mathbf{x})$ denotes a diagonal matrix with entries given by the vector $\mathbf{x}=[x_1,x_2,\dots,x_n]^{\top} \in \mathbb{R}^N$.
Finally, $\Sp(\mathbf{A})$ denotes the set of eigenvalues of a matrix $\mathbf{A}$.

\vspace{0.1cm}

\noindent \textbf{Basic Definitions}.
A graph $G$ is defined as a pair $(\mathcal{V}, \mathcal{E})$, where $\mathcal{V}$ is a set of vertices and $\mathcal{E} \subseteq \mathcal{V} \times \mathcal{V}$ is a set of edges.
Each edge $e \in \mathcal{E}$ is a pair of vertices $(u, v)$, representing a connection between nodes $u$ and $v$. 
A bipartite graph $B$ is a special case of a graph, defined as $(\mathcal{V}_L, \mathcal{V}_R, \mathcal{E})$, where $\mathcal{V}_L$ and $\mathcal{V}_R$ are disjoint sets of vertices, and $\mathcal{E} \subseteq \mathcal{V}_L \times \mathcal{V}_R$.
Every edge in a bipartite graph connects a node in $\mathcal{V}_L$ to a node in $\mathcal{V}_R$.
Note that a bipartite graph can be identified as a graph where $\mathcal{V} = \mathcal{V}_L \cup \mathcal{V}_R$.
We define the Laplacian of a graph $\mathbf{L}_G \in \mathbb{R}^{|\mathcal{V}| \times |\mathcal{V}|}$ as $\mathbf{L}_G = \mathbf{D} - \mathbf{W}$, where $\mathbf{D}$ is the diagonal degree matrix and $\mathbf{W}$ is the weighted adjacency matrix with $\mathbf{W}_{[i,j]} \neq 0$ if $(i,j) \in \mathcal{E}$.
The normalized Laplacian is defined as $\boldsymbol{\mathcal{L}}_G = \mathbf{I} - \mathbf{D}^{-1/2} \mathbf{W} \mathbf{D}^{-1/2}$, where $\mathbf{I}$ is the identity.

\begin{figure*}[t]
    \centering
    \includegraphics[width=0.96\linewidth]{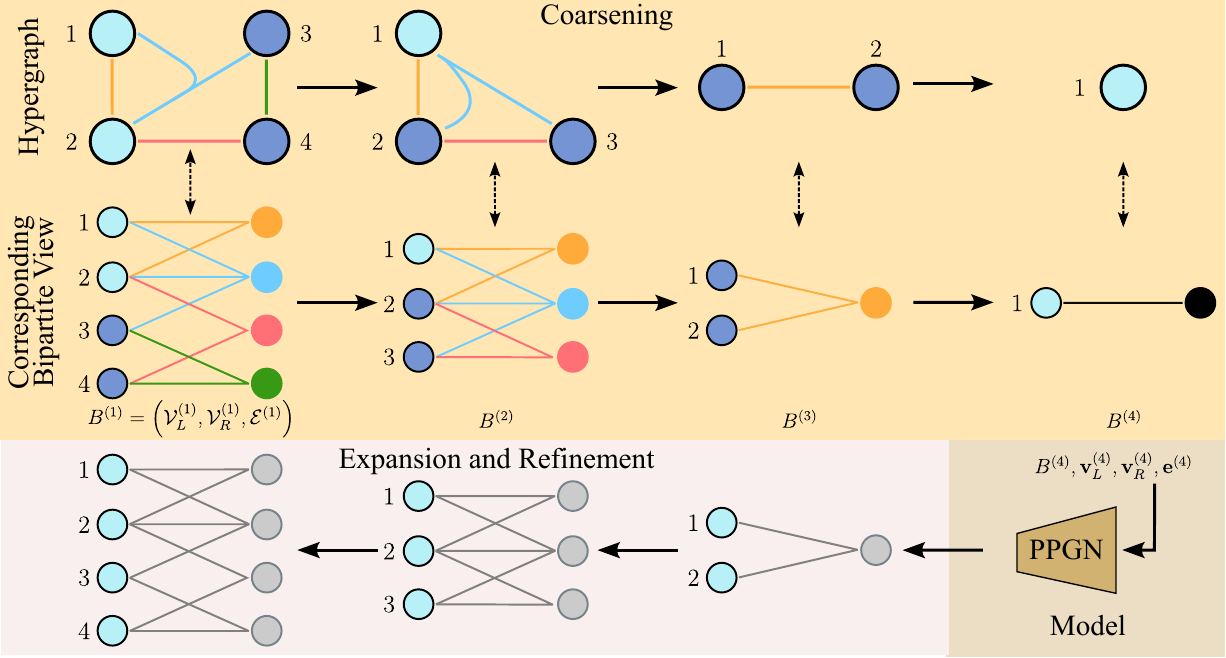}
    \caption{Starting from a hypergraph, our method computes different views of it at increasingly coarser resolutions. An equivalent bipartite representation is maintained in parallel, and a graph neural network model (in our case a PPGN \cite{maron2020provablypowerfulgraphnetworks}) is trained to recover a bipartite representation from its coarser version.}
    \label{fig:pipeline}
\end{figure*}

A hypergraph $H$ is defined as a pair $(\mathcal{V}, \mathcal{E})$, where $\mathcal{V}$ is a set of vertices and $\mathcal{E}$ is a set of hyperedges, with each $e \in \mathcal{E}$ being a subset of $\mathcal{V}$.
Unlike in graphs, hyperedges can connect any number of vertices.
We define two graph representations of hypergraphs: the clique and star expansions.
The clique expansion of a hypergraph $H$ is a graph $C = (\mathcal{V}_c, \mathcal{E}_c)$, where $\mathcal{E}_c = \{(u, v) \mid \exists~e \in \mathcal{E} : u, v \in e\}$.
The star expansion of a hypergraph $H$ is a bipartite graph $B = (\mathcal{V}_L, \mathcal{V}_R, \mathcal{E}_b)$, where $\mathcal{V}_L = \mathcal{V}$, $\mathcal{V}_R = \mathcal{E}$, and $\mathcal{E}_b = \{(v, e) \mid v \in \mathcal{V}_L, e \in \mathcal{V}_R, v \in e \text{ in } H\}$.
Furthermore, we define the Bolla's Laplacian \cite{BOLLA199319} of a hypergraph as $\mathbf{L}_H = \mathbf{D}_\mathcal{V} - \mathbf{H} \mathbf{D}_\mathcal{E}^{-1} \mathbf{H}^T$, where $\mathbf{D}_\mathcal{V} \in \mathbb{N}^{|\mathcal{V}| \times |\mathcal{V}|}$ is the diagonal degree matrix for the nodes, $\mathbf{D}_\mathcal{E} \in \mathbb{N}^{|\mathcal{E}| \times |\mathcal{E}|}$ is the diagonal matrix of edge orders, and $\mathbf{H} \in \{0, 1\}^{|\mathcal{V}| \times |\mathcal{E}|}$ is the incidence matrix.
The normalized version of $\mathbf{L}_H$, known as Zhou's Laplacian \cite{Zhou2005}, is defined as $\boldsymbol{\mathcal{L}}_H = \mathbf{I} - \mathbf{D}_\mathcal{V}^{-1/2} \mathbf{H} \mathbf{D}_\mathcal{E}^{-1} \mathbf{H}^T \mathbf{D}_\mathcal{V}^{-1/2}$.

The goal of this work is to train a model capable of sampling from the underlying distribution of a given dataset of hypergraphs $(H_1, \dots, H_N)$, \ie learning to generate hypergraphs from data.
All the proofs of propositions and lemmas of this paper are provided in Appendix \ref{sec::proofs}.

\subsection{Overview}
\label{sec:overview}
The workflow of our approach is illustrated in Figure \ref{fig:pipeline}.
Our method uses two distinct representations of hypergraphs: the weighted clique expansion and the star expansion.
The weighted clique expansion (not depicted in the figure) facilitates the downward traversal of the resolution scales.
This representation enables the application of the algorithm proposed by \citet{loukas2018graph} to generate reduced versions of the hypergraph while maintaining its spectral properties.
Concurrently, we maintain the hypergraph's star expansion, representing it as a bipartite graph.
In this representation, one partition corresponds to the hypergraph's nodes (left-side), while the other represents its hyperedges (right-side).
Each node is connected to the hyperedges containing it.
The bipartite representation proves particularly convenient for ascending the resolution scales.
Starting from a coarser representation, the model learns to identify merged nodes and hyperedges, subsequently reconstructing the original bipartite representation.
In our implementation, we employ the denoising diffusion model framework \cite{karras2022elucidatingdesignspacediffusionbased} to model the learning problem: the truth values for the nodes and edges requiring expansion and deletion are noised, and a model is trained to recover the original values.
We chose Provably Powerful Graph Network (PPGN) \cite{maron2020provablypowerfulgraphnetworks} as the architecture of this model.
Hypergraph generation can then be achieved through an iterative process of increasing resolution, starting from the coarsest bipartite representation, which consists of a pair of connected nodes.

\subsection{Descending through the Resolution Scales: Coarsening Sequences}

\label{sec:descending}
\begin{figure*}[t]
    \centering
    \begin{subfigure}[b]{0.29\textwidth}
        \centering
        \includegraphics[width=\textwidth]{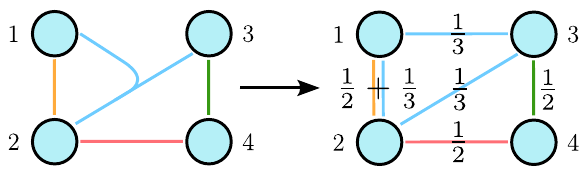}
        \caption{Weighted clique expansion}
        \label{fig:weighted_clique_expansion}
    \end{subfigure}
    \hfill
    \begin{subfigure}[b]{0.29\textwidth}
        \centering
        \includegraphics[width=\textwidth]{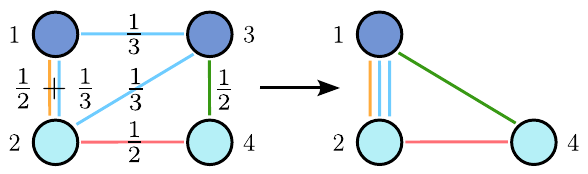}
        \caption{Coarsening}
        \label{fig:clique_coarsening}
    \end{subfigure}
    \hfill
    \begin{subfigure}[b]{0.4\textwidth}
        \centering
        \includegraphics[width=\textwidth]{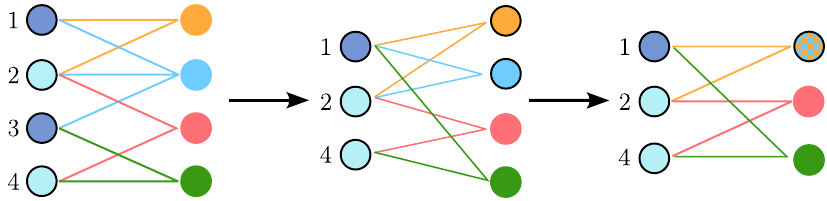}
        \caption{Bipartite representation update}
        \label{fig:bipartite_coarsening}
    \end{subfigure}
    \caption{Coarsening: (a) Compute the weighted clique expansion by collapsing each hyperedge into an appropriately weighted clique. (b) Coarsen the clique expansion while preserving the spectral properties of the hypergraph (dark blue nodes). (c) Update the bipartite view: corresponding left-side nodes (in dark blue) are merged, then right-side nodes representing the same hyperedge (circled in black) are merged.}
    \label{fig:three_coarsening_steps}
\end{figure*}


\begin{definition}[Graph coarsening] \label{def::coarsening}
Let $G = (\mathcal{V}, \mathcal{E})$ be an arbitrary graph and $\mathcal{P} = \{\mathcal{V}^{(1)}, \ldots, \mathcal{V}^{(\bar{n})}\}$ be a partitioning\footnote{{This implies that $\mathcal{V}^{(p)} \subseteq \mathcal{V}$, $\bigcup_{i=1}^{\bar{n}} \mathcal{V}^{(i)} = \mathcal{V}$, and $~{\mathcal{V}^{(i)} \cap \mathcal{V}^{(j)} = \emptyset~\forall~1\leq i,j \leq \bar{n}}$.}}  of the node set $\mathcal{V}$ such that each set $\mathcal{V}^{(p)} \in \mathcal{P}$ induces a connected subgraph in $G$.
We construct a coarsening $\bar{G}(G, \mathcal{P}) = (\bar{\mathcal{V}}, \bar{\mathcal{E}})$ of $G$ by representing each part $\mathcal{V}^{(p)} \in \mathcal{P}$ as a single node $v^{(p)} \in \bar{\mathcal{V}}$.
We add an edge $e_{\{p,q\}} \in \bar{\mathcal{E}}$, between distinct nodes $v^{(p)} \neq v^{(q)} \in \bar{\mathcal{V}}$ in the coarsened graph if and only if there exists an edge $e_{\{i,j\}} \in \mathcal{E}$ between the corresponding disjoint clusters in the original graph, \ie $v^{(i)} \in \mathcal{V}^{(p)}$ and $v^{(j)} \in \mathcal{V}^{(q)}$.
\end{definition}

\begin{remark}
This definition implies partitioning the nodes into different connected sets and merging each part into a cluster.
Two clusters are connected if and only if there exists an edge between some node in the first cluster and some node in the second cluster.
\end{remark}

We extend Definition \ref{def::coarsening} to the bipartite representations of hypergraphs:

\begin{definition}[Bipartite representation coarsening] \label{def::bipartite_coarsening}
Let $H$ be an arbitrary hypergraph, $C = (\mathcal{V}_c, \mathcal{E}_c)$ its weighted clique expansion, $B = (\mathcal{V}_L, \mathcal{V}_R, \mathcal{E})$ its bipartite representation, and $\mathcal{P}_L = \{\mathcal{V}^{(1)}, \ldots, \mathcal{V}^{(n)}\}$ a partitioning of the node set $\mathcal{V}_c$ of $C$ (and equivalently of the node set $\mathcal{V}_L$ of $B$), such that each part $\mathcal{V}^{(p)} \in \mathcal{P}_L$ induces a connected subgraph in $C$.
We construct an intermediate coarsening $\bar{B}(B, \mathcal{P}_L) = (\bar{\mathcal{V}_L}, \mathcal{V}_R, \bar{\mathcal{E}})$ of $B$ by representing each part $\mathcal{V}^{(p)} \in \mathcal{P}_L$ as a single node $v^{(p)} \in \bar{\mathcal{V}_L}$. We add an edge $e_{\{p,q\}} \in \bar{\mathcal{E}}$, between distinct nodes $v^{(p)} \in \bar{\mathcal{V}_L}$ and $v^{(q)} \in \mathcal{V}_R$ in the coarsened graph if and only if there exists an edge $e_{\{i, q\}} \in \mathcal{E}$ between a node $v^{(i)} \in \mathcal{V}^{(p)}$ and right-side node $v^{(q)}$ in the original graph.

Now let $v_1 \sim v_2 \iff \mathcal{N}(v_1) = \mathcal{N}(v_2)$ define an equivalence relation for the right-side nodes in $\mathcal{V}_R$, where $\mathcal{N}(v)$ denotes the set of neighbors of $v$.
This equivalence relation induces a partitioning $\mathcal{P}_R = \{\mathcal{V}_R^{(1)}, \dots, \mathcal{V}_R^{(m)}\}$ of $\mathcal{V}_R$. 
Finally, we construct the fully coarsened bipartite representation by representing each part $\mathcal{V}^{(p)} \in \mathcal{P}_R$ as a single node $v^{(p)} \in \bar{\mathcal{V}_R}$, in a similar way to $\bar{\mathcal{V}_L}$.
\end{definition}

\begin{remark}
Here, nodes are merged into clusters, and then hyperedges appearing multiple times are merged. This process is illustrated in Figure \ref{fig:three_coarsening_steps}.
\end{remark}

We now describe the construction of coarsening sequences for a hypergraph $H = (\mathcal{V}, \mathcal{E})$. This process maintains the weighted clique expansion and the bipartite representation of $H$ in parallel. We leverage the following result:

\begin{proposition}[Adapted from Section 6.4 in \cite{HigherOrderLearning} and proved in Appendix \ref{sec::spectral_equivalence_bipartite}]
    For an unweighted hypergraph, Bolla's \textit{unnormalized} Laplacian $\mathbf{L}_H = \mathbf{D}_\mathcal{V} - \mathbf{H} \mathbf{D}_\mathcal{E}^{-1} \mathbf{H}^T$ is equal to the \textit{unnormalized} Laplacian of the associated clique expansion $C$, where each edge $e_{uv}$ is weighted by $\sum_{e \ni u, v ; \text{ } e \in \mathcal{E}} \frac{1}{|e|}$.
\end{proposition}

This proposition establishes that the spectral properties of the hypergraph are equivalent to those of an appropriately weighted clique expansion.
\citet{loukas2018graph} introduced an algorithm to construct a coarsened version of a graph while preserving a subset of its eigenvalues and eigenvectors.
As the weighted clique expansion is a graph, this algorithm can be applied to select groups of nodes for merging.
This allows the construction of a coarser view of the hypergraph while preserving relevant spectral properties, which are known to capture important characteristics of the underlying topology.
Figure \ref{fig:weighted_clique_expansion} illustrates an example of such a weighted clique expansion.

The coarsening process consists of three steps (illustrated in Figures \ref{fig:clique_coarsening} and \ref{fig:bipartite_coarsening}):
\begin{enumerate}
    \item The algorithm by \citet{loukas2018graph} operates on the weighted clique expansion to identify a suitable partitioning of nodes, also referred to as ``contraction sets".
    \item These contraction sets are merged in the weighted clique expansion, and the corresponding left-side nodes in the bipartite representation of the hypergraph are also merged.
    \item Finally, the right-side nodes in the bipartite representation of the hypergraph representing the same hyperedge are merged.
\end{enumerate}

This procedure is applied iteratively until we obtain a single-node graph for the weighted clique expansion and a corresponding bipartite graph with a single node on each side for the bipartite representation.
It is important to note that in this setting, controlling the merging of hyperedges (right-side nodes for the bipartite representation) is challenging.
Empirical experiments have shown that even with few left-node mergings, the right side can easily merge tens of hyperedges into one cluster at once in dense hypergraphs.
To avoid this issue, we select the contraction family to be the set of all pairs of adjacent nodes in the clique representation.
Therefore, we obtain the following result:
\begin{proposition}
\label{prop:upper_bound_edges}
    For a single merging of two adjacent nodes in the clique representation, at most three hyperedges can be involved in each hyperedge merging in the bipartite representation.
\end{proposition}


\begin{remark} \label{rem::our_setting}
    This proposition holds only for a single merging of a node pair at each coarsening step. In the case of multiple simultaneous mergings, the proposition does not necessarily hold. However, it can be enforced by ensuring that the different contraction sets have disjoint neighborhoods. In our experiments, we instead consider each relevant node merging individually and proceed with it only if every right-side cluster does not exceed three hyperedges. The complete coarsening sampling procedure incorporating this approach is detailed in Algorithm \ref{alg::hypergraph_coarsening} of Appendix \ref{sec::coarsening}.
\end{remark}

\subsection{Ascending through the Resolution Scales: Expansion and Refinement}
\label{sec:ascending}

\begin{figure}[t]
    \centering
    \includegraphics[width=\columnwidth]{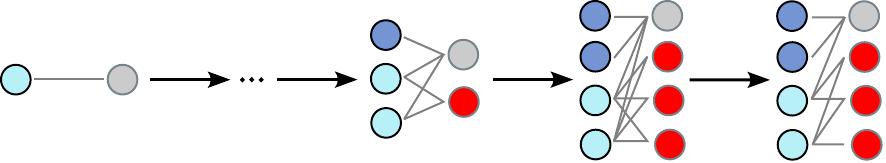}
    \caption{Our model starts from a single pair of linked nodes (in the bipartite representation) and iteratively expands the left-side nodes (in dark blue) and right-side nodes (in red), where each duplicate keeps the connections of its parent node. Then, our method refines the resulting bipartite graph filtering edges to recover an appropriate local structure.}
    \label{fig:hypergraph_generation}
\end{figure}

We now describe the expansion and refinement of the bipartite representation of a hypergraph, which is the inverse of the coarsening process (see Appendix \ref{sec::proof_inversion} for a proof).
This process operates exclusively on the bipartite representation.
At each step, starting from the bipartite representation for a specific resolution level $B = (\mathcal{V}_L, \mathcal{V}_R, \mathcal{E})$, we first select which nodes to duplicate.
This selection is encoded via two vectors: $\mathbf{v}_L \in \mathbb{N}^{|\mathcal{V}_L|}$ for left-side nodes, and $\mathbf{v}_R \in \mathbb{N}^{|\mathcal{V}_R|}$ for right-side nodes.
These vectors specify the number of times each node needs to be duplicated.
Therefore, we expand the bipartite graph by duplicating each node by the specified number.
Each duplicate retains the same connectivity as its parent cluster node.

The following definition formalizes this process (illustrated in the two center figures of Figure \ref{fig:hypergraph_generation}):

\begin{definition}[Bipartite graph expansion] \label{def::expansion}
    Given a bipartite graph $B = (\mathcal{V}_L, \mathcal{V}_R, \mathcal{E})$ and two cluster size vectors $\mathbf{v}_L \in \mathbb{N}^{|\mathcal{V}_L|}$, $\mathbf{v}_R \in \mathbb{N}^{|\mathcal{V}_R|}$, denoting the expansion size of each node, let $\tilde{B}(B, \mathbf{v}_L, \mathbf{v}_R) = (\tilde{\mathcal{V}}_L, \tilde{\mathcal{V}}_R, \tilde{\mathcal{E}})$ denote the expansion of $B$, whose node sets are given by:
    \begin{itemize}
        \item $\tilde{\mathcal{V}}_L = \mathcal{V}_L^{(1)} \cup \cdots \cup \mathcal{V}_L^{(|\mathcal{V}_L|)}$, where $\mathcal{V}_L^{(p)} = \{v_L^{(p, i)} \mid 1 \leq i \leq \mathbf{v}_L[p]\}$ for $1 \leq p \leq |\mathcal{V}_L|$.
        \item $\tilde{\mathcal{V}}_R = \mathcal{V}_R^{(1)} \cup \cdots \cup \mathcal{V}_R^{(|\mathcal{V}_R|)}$, where $\mathcal{V}_R^{(p)} = \{v_R^{(p, i)} \mid 1 \leq i \leq \mathbf{v}_R[p]\}$ for $1 \leq p \leq |\mathcal{V}_R|$.
    \end{itemize}
    The edge set $\tilde{\mathcal{E}}$ includes all the cluster interconnecting edges:
    $\{e_{\{p, i; q, j\}} \mid e_{\{p, q\}} \in \mathcal{E}, v_L^{(p, i)} \in \mathcal{V}_L^{(p)}, v_R^{(q, j)} \in \mathcal{V}_R^{(q)}\}$.
\end{definition}

After the expansion, we selectively keep or remove edges of the resulting bipartite graph using an edge selection vector $\mathbf{e} \in \{0, 1\}^{|\mathcal{E}|}$ (this corresponds to the rightmost step in Figure \ref{fig:hypergraph_generation}):

\begin{definition}[Bipartite representation refinement] \label{def::refinement}
    Given a bipartite graph $\tilde{B} = (\tilde{\mathcal{V}}_L, \tilde{\mathcal{V}}_R, \tilde{\mathcal{E}})$ and an edge selection vector $\mathbf{e} \in \{0, 1\}^{|\tilde{\mathcal{E}}|}$, let $B(\tilde{B}, \mathbf{e}) = (\mathcal{V}_L, \mathcal{V}_R, \mathcal{E})$ denote the refinement of $\tilde{B}$, where: $\mathcal{V}_L = \tilde{\mathcal{V}}_L$, $\mathcal{V}_R = \tilde{\mathcal{V}}_R$, $\mathcal{E} \subseteq \tilde{\mathcal{E}}$ such that the $i$-th edge $e_{(i)} \in \mathcal{E}$ if and only if $\mathbf{e}[i] = 1$.
\end{definition}

The process of generating a hypergraph consists of the following steps:
\begin{enumerate}
    \item Start from a bipartite graph containing only two nodes linked by an edge: $B^{(L)} = ( \{1\}, \{2\}, \{(1, 2)\} )$ (leftmost figure of Figure \ref{fig:hypergraph_generation}).
    \item Iteratively expand and refine the current bipartite representation to add details until the desired size is attained:
        $$B^{(l)} \xrightarrow{\text{expand}} \tilde{B}^{(l-1)} \xrightarrow{\text{refine}} B^{(l-1)}$$
    \item Once the final bipartite representation is generated, recover the associated hypergraph by collapsing each node on the right-side into a hyperedge.
\end{enumerate}

\begin{remark}
    Our generation part differs from that by \citet{bergmeister2024efficient} in the expansion step. 
    While they initially retain all possible edges between connected clusters before selection, we treat hyperedges analogously to nodes. 
    This constraint is necessary to maintain computational feasibility due to the exponential growth of potential hyperedges in hypergraphs.
\end{remark}

\subsection{Probabilistic Modeling}
\label{sec:prob}
We now formalize our learning problem, generalizing the approach by \citet{bergmeister2024efficient}. Given a dataset $\{H^{(1)}, \ldots, H^{(N)}\}$ of i.i.d. hypergraph samples, our goal is to fit a distribution $p(H)$ that closely approximates the unknown true generative process. We model the marginal likelihood of a hypergraph $H$ as the sum of likelihoods over expansion sequences of its bipartite representation $B$:
\begin{equation}
p(H) = p(B) = \sum_{\varpi \in \Pi(B)} p(\varpi),
\end{equation}
where $\Pi(B)$ denotes the set of all possible expansion sequences $(B^{(L)} = (\{1\}, \{2\}, \{(1, 2)\}), B^{(1)}, \ldots, B^{(0)} = B)$ from a minimal bipartite graph to the target hypergraph's bipartite representation $B$.
Each $B^{(l-1)}$ is a refined expansion of its predecessor, as per Definitions \ref{def::expansion} and \ref{def::refinement}.

\begin{table*}[t]
    \centering
    \setlength{\tabcolsep}{3pt}  
    \resizebox{\textwidth}{!}{%
    \begin{tabular}{@{}r*{5}{c}*{5}{c}*{5}{c}@{}}
    \toprule
    \multirow{4}{*}{\textbf{Model}} & \multicolumn{5}{c}{\textbf{SBM Hypergraphs}} & \multicolumn{5}{c}{\textbf{Ego Hypergraphs}} & \multicolumn{5}{c}{\textbf{Tree Hypergraphs}} \\
    & \multicolumn{5}{c}{($n_{avg} = $ 31.73, $std = $ 0.55)} &\multicolumn{5}{c}{($n_{avg} = $ 109.71, $std = $ 10.23)} &\multicolumn{5}{c}{($n_{avg} = $ 32, $std = $ 0)} \\
    \cmidrule(l){2-6}
    \cmidrule(l){7-11}
    \cmidrule(l){12-16}
    &
    \makecell{\textbf{Valid} \\ \textbf{SBM} $\uparrow$} &
    \makecell{\textbf{Node} \\ \textbf{Num} $\downarrow$} &
    \makecell{\textbf{Node} \\ \textbf{Deg} $\downarrow$} &
    \makecell{\textbf{Edge} \\ \textbf{Size} $\downarrow$} &
    \textbf{Spectral $\downarrow$} &
    \makecell{\textbf{Valid} \\ \textbf{Ego} $\uparrow$} &
    \makecell{\textbf{Node} \\ \textbf{Num} $\downarrow$} &
    \makecell{\textbf{Node} \\ \textbf{Deg} $\downarrow$} &
    \makecell{\textbf{Edge} \\ \textbf{Size} $\downarrow$} &
    \textbf{Spectral $\downarrow$} &
    \makecell{\textbf{Valid} \\ \textbf{Tree} $\uparrow$} &
    \makecell{\textbf{Node} \\ \textbf{Num} $\downarrow$} &
    \makecell{\textbf{Node} \\ \textbf{Deg} $\downarrow$} &
    \makecell{\textbf{Edge} \\ \textbf{Size} $\downarrow$} &
    \textbf{Spectral $\downarrow$} \\
    \midrule
    HyperPA &2.5\% &\textbf{0.075} &4.062 &0.407 &0.273 & 0\% &35.83 &2.590 &0.423 &0.237 &0\% &2.350 &0.315 &0.284 &0.159 \\
    VAE &0\% &0.375 &1.280 &1.059 &0.024 &0\% &47.58 &0.803 &1.458 &0.133 &0\% &9.700 &0.072 &0.480 &0.124 \\
    GAN &0\% &1.200 &2.106 &1.203 &0.059 & 0\% &60.35 &0.917 &1.665 &0.230 &0\% &6.000 &0.151 &0.469 &0.089 \\
    Diffusion &0\% &0.150 &1.717 &1.390 &0.031 & 0\% &\textbf{4.475} &3.984 &2.985 &0.190 &0\% &2.225 &1.718 &1.922 &0.127 \\[0.2em]
    \cdashline{1-16}\\[-0.8em]
    \method &\textbf{65}\%& 0.525& \textbf{0.321}& \textbf{0.002}& \textbf{0.010} & \textbf{90}\%& 12.55& \textbf{0.063}& \textbf{0.220}& \textbf{0.004} &\textbf{77.5}\%& \textbf{0.000}& \textbf{0.059}& \textbf{0.108} & \textbf{0.012} \\
    \bottomrule
    \end{tabular}%
    }
    \caption{Comparison between \method~and other baselines for the SBM, Ego, and Tree hypergraphs.}
    \label{tab:hypergraph_SBM_Ego_Tree_results}
\end{table*}

Assuming a Markovian structure, we factorize the likelihood as:
\begin{align}
p(\varpi) &= \underbrace{p(B^{(L)})}_{1} \cdot \prod_{l=L}^1 p(B^{(l-1)} | B^{(l)}) \\
&= \prod_{l=L}^1 p(\mathbf{e}^{(l-1)} | \tilde{B}^{(l-1)})p(\mathbf{v}_L^{(l)}, \mathbf{v}_R^{(l)} | B^{(l)}).
\end{align}
To avoid modeling two separate distributions $p(\mathbf{e}^{(l)} | \tilde{B}^{(l)})$ and $p(\mathbf{v}_L^{(l)}, \mathbf{v}_R^{(l)} | B^{(l)})$, we rearrange terms as:
\begin{equation}
\begin{split}
    p(\varpi) = \enspace &p(\mathbf{e}^{(0)} | \tilde{B}^{(0)}) \cdot \underbrace{p(\mathbf{v}_L^{(L)}, \mathbf{v}_R^{(L)} | B^{(L)})}_{p(\mathbf{v}_L^{(L)}, \mathbf{v}_R^{(L)})} \cdot \\
    &\cdot \left[ \prod_{l=L-1}^1 p(\mathbf{v}_L^{(l)}, \mathbf{v}_R^{(l)} | B^{(l)})p(\mathbf{e}^{(l)} | \tilde{B}^{(l)}) \right].
\label{eq:1}
\end{split}
\end{equation}
We model $\mathbf{v}_L^{(l)}$ and $\mathbf{v}_R^{(l)}$ to be conditionally independent of $\tilde{B}^{(l)}$ given $B^{(l)}$, \ie $p(\mathbf{v}_L^{(l)}, \mathbf{v}_R^{(l)} | B^{(l)}, \tilde{B}^{(l)}) = p(\mathbf{v}_L^{(l)}, \mathbf{v}_R^{(l)} | B^{(l)})$. This allows us to finally write:
\begin{equation}
    p(\mathbf{v}_L^{(l)}, \mathbf{v}_R^{(l)} | B^{(l)})p(\mathbf{e}^{(l)} | \tilde{B}^{(l)}) = p(\mathbf{v}_L^{(l)}, \mathbf{v}_R^{(l)}, \mathbf{e}^{(l)} | \tilde{B}^{(l)}).
\end{equation}

\subsection{Implementation}

We employ the EDM denoising diffusion framework \cite{karras2022elucidatingdesignspacediffusionbased} to model the probability $p(\mathbf{v}_L, \mathbf{v}_R, \mathbf{e} | B)$:
$\mathbf{v}_L$, $\mathbf{v}_R$, and $\mathbf{e}$ are treated as node and edge features of the bipartite representation; these features are noised and a model is trained to recover the original features.
Our model consists of an embedding layer for each vector, followed by a PPGN \cite{maron2020provablypowerfulgraphnetworks} feeding into one final output layer.
See Appendix \ref{sec::model_architecture} for details of the architecture.

The main implementation details are outlined below, with further details provided in Appendix \ref{sec::miscellaneous}.
\begin{enumerate}
    \item Similarly to \cite{bergmeister2024efficient}, we use a deterministic expansion size procedure where only a predefined number of nodes are expanded at each iteration, those being the most probable according to the model, while the others are left unchanged.
    \item We also reuse the perturbed expansion, where the bipartite graph is noised through the introduction of random edges connecting a node to others located within a predefined radius around it.
    \item We extend spectral conditioning to hypergraphs, wherein spectral properties of the target hypergraph are used as conditioning during the prediction of $B^{(l)}$.
\end{enumerate}
Indeed, the spectrum of $B^{(l+1)}$ approximates the target spectrum due to the spectral preservation during coarsening: using Lemma 1 by \citet{HigherOrderLearning}, we can easily prove (see Appendix \ref{sec::spectral_equivalence_bipartite}):
\begin{equation} \label{eq::link_spectra}
    Sp(\mathbf{\mathcal{L}}_B) = \left\{ 1 \pm \sqrt{1 - \lambda} \: \vert \: \lambda \in Sp(\mathbf{\mathcal{L}}_H) \right\} \subset [0, 2],
\end{equation}
where $Sp(\mathbf{\mathcal{L}}_B)$ and $Sp(\mathbf{\mathcal{L}}_H)$ are the spectra of the \textit{normalized} Laplacians of the bipartite representation and the hypergraph, respectively. 
There is also equivalence for the eigenspaces (see the proof in Appendix \ref{sec::spectral_equivalence_bipartite}).
Consequently, preserving the $k$ smallest non-zero eigenvalues of the \textit{unnormalized} Laplacian of the weighted clique expansion also preserves the $k$ smallest non-zero (and $k$ largest non-equal to 2) eigenvalues of the normalized Laplacian of the bipartite representation.

During hypergraph reconstruction from its bipartite representation, isolated nodes and empty hyperedges (corresponding to empty rows and columns in the incidence matrix) are discarded.

\section{Experiments and Results} \label{sec::experiments}

In this section, we detail our experimental setup, covering datasets and evaluation metrics.
Then, we compare our approach against the following baselines: HyperPA~\citep{Do_2020}, a Variational Autoencoder (VAE)~\cite{kingma2022autoencodingvariationalbayes}, a Generative Adversarial Network (GAN)~\cite{goodfellow2014generativeadversarialnetworks}, and a standard 2D diffusion model~\cite{ho2020denoisingdiffusionprobabilisticmodels} trained on incidence matrix images, where hyperedge membership is represented by white pixels and absence by black pixels.
Finally, we ablate on the spectrum-preserving coarsening and the upper bound for the number of hyperedges defined in Proposition~\ref{prop:upper_bound_edges}.

Our goal is threefold: (i) proving that \method~can generate the desired hyperedge distribution, (ii) proving that \method~can closely mimic a range of strict structural properties, and (iii) proving the importance of our components, \ie spectrum-preserving coarsening and upper bounding the size of hyperedge clusters during coarsening.
Detailed numerical results are available in Appendix~\ref{sec::detailed_results}, and Appendix~\ref{app:visual_examples} provides several visualizations of our generated hypergraphs.

\vspace{0.1cm}

\noindent \textbf{Datasets}.
We evaluate our method on four synthetic hypergraph datasets: Erdős–Rényi (ER)~\cite{ErdösRényi+2006+38+82}, Stochastic Block Model (SBM)~\cite{kim2018stochasticblockmodelhypergraphs}, Ego~\cite{comrie2021hypergraphegonetworkstemporalevolution}, and Tree~\cite{NIEMINEN199935}.
Furthermore, we also test \method~on topologies of low-poly featureless versions of three classes of ModelNet40~\cite{wu20153dshapenetsdeeprepresentation} converted to hypergraphs: \textit{plant}, \textit{piano}, and \textit{bookshelf}.
Each dataset is split into $128$ training, $32$ validation, and $40$ test hypergraphs.
More details can be found in Appendix~\ref{app:datasets}.

\vspace{0.1cm}

\begin{table*}[t]
    \centering
    \setlength{\tabcolsep}{3pt}  
    \resizebox{\textwidth}{!}{%
    \begin{tabular}{@{}r*{5}{c}*{5}{c}*{5}{c}*{5}{c}@{}}
    \toprule
    \multirow{4}{*}{\textbf{Model}} & \multicolumn{4}{c}{\textbf{Erdos-Renyi Hypergraphs}} & \multicolumn{4}{c}{\textbf{ModelNet40 Piano}} & \multicolumn{4}{c}{\textbf{ModelNet40 Plant}} & \multicolumn{4}{c}{\textbf{ModelNet40 Bookshelf}} \\
    & \multicolumn{4}{c}{($n_{avg} = $ 32, $std = $ 0.07)} &\multicolumn{4}{c}{($n_{avg} = $ 177.29, $std = $ 57.11)} &\multicolumn{4}{c}{($n_{avg} = $ 124.86, $std = $ 87.88)} &\multicolumn{4}{c}{($n_{avg} = $ 119.38, $std = $ 68.20)} \\
    \cmidrule(l){2-5}
    \cmidrule(l){6-9}
    \cmidrule(l){10-13}
    \cmidrule(l){14-18}
    &
    \makecell{\textbf{Node} \\ \textbf{Num} $\downarrow$} &
    \makecell{\textbf{Node} \\ \textbf{Deg} $\downarrow$} &
    \makecell{\textbf{Edge} \\ \textbf{Size} $\downarrow$} &
    \textbf{Spectral $\downarrow$} &
    \makecell{\textbf{Node} \\ \textbf{Num} $\downarrow$} &
    \makecell{\textbf{Node} \\ \textbf{Deg} $\downarrow$} &
    \makecell{\textbf{Edge} \\ \textbf{Size} $\downarrow$} &
    \textbf{Spectral $\downarrow$} &
    \makecell{\textbf{Node} \\ \textbf{Num} $\downarrow$} &
    \makecell{\textbf{Node} \\ \textbf{Deg} $\downarrow$} &
    \makecell{\textbf{Edge} \\ \textbf{Size} $\downarrow$} &
    \textbf{Spectral $\downarrow$} &
    \makecell{\textbf{Node} \\ \textbf{Num} $\downarrow$} &
    \makecell{\textbf{Node} \\ \textbf{Deg} $\downarrow$} &
    \makecell{\textbf{Edge} \\ \textbf{Size} $\downarrow$} &
    \textbf{Spectral $\downarrow$} \\
    \midrule
    HyperPA &\textbf{0.000} &5.530 &0.183 &0.177 &0.825 &9.254 &\textbf{0.023} &\textbf{0.067} &10.83 &6.566 &0.046 &0.061 &8.025 &7.562 &0.044 &\textbf{0.048} \\
    VAE &0.100 &2.140 &0.540 &0.035 &75.35 &8.060 &1.686 &0.396 &76.15 &3.895 &1.573 &0.205 &47.45 &6.190 &1.520 &0.190 \\
    GAN &0.675 &2.560 &0.657 &0.048 &\textbf{0.000} &409.0 &86.38 &0.697 &\textbf{0.000} &378.1 &56.35 &0.364 &\textbf{0.000} &397.2 &46.30 &0.476 \\
    Diffusion &0.050 &2.225 &0.781 &0.014 &0.050 &20.90 &4.192 &0.113 &0.025 &21.03 &3.439 &0.069 &\textbf{0.000} &20.36 &2.346 &0.079 \\[0.2em]
    \cdashline{1-18}\\[-0.8em]
    \method &0.775& \textbf{0.475}& \textbf{0.012}& \textbf{0.006} &42.52 &\textbf{6.290} &0.027 &0.117 &68.38 &\textbf{2.428} &\textbf{0.027} &\textbf{0.034} &69.73 &\textbf{1.050} &\textbf{0.034} &0.068 \\
    \bottomrule
    \end{tabular}%
    }
    \caption{Comparison between \method~and other baselines for the ER and ModelNet40 hypergraphs.}
    \label{tab:hypergraph_results_er_modelnet40}
\end{table*}

\begin{table*}[t]
    \centering
    \setlength{\tabcolsep}{3pt}  
    \resizebox{\textwidth}{!}{%
    \begin{tabular}{@{}*{2}{c}*{4}{c}*{4}{c}*{4}{c}*{3}{c}@{}}
    \toprule
    \multirow{3}{*}{\makecell{\textbf{Upper} \\ \textbf{Bound}}} & \multirow{3}{*}{\makecell{\textbf{Spec. Pr.} \\ \textbf{Coarsen.}}} & \multicolumn{4}{c}{\textbf{SBM Hypergraphs}} & \multicolumn{4}{c}{\textbf{Ego Hypergraphs}} & \multicolumn{4}{c}{\textbf{Tree Hypergraphs}} & \multicolumn{3}{c}{\textbf{Erdős–Rényi Hypergraphs}} \\
    \cmidrule(l){3-6}
    \cmidrule(l){7-10}
    \cmidrule(l){11-14}
    \cmidrule(l){15-17}
    & & 
    \makecell{\textbf{Valid} \\ \textbf{SBM} $\uparrow$} &
    \makecell{\textbf{Node} \\ \textbf{Deg} $\downarrow$} &
    \makecell{\textbf{Edge} \\ \textbf{Size} $\downarrow$} &
    \textbf{Spectral $\downarrow$} &
    \makecell{\textbf{Valid} \\ \textbf{Ego} $\uparrow$} &
    \makecell{\textbf{Node} \\ \textbf{Deg} $\downarrow$} &
    \makecell{\textbf{Edge} \\ \textbf{Size} $\downarrow$} &
    \textbf{Spectral $\downarrow$} &
    \makecell{\textbf{Valid} \\ \textbf{Tree} $\uparrow$} &
    \makecell{\textbf{Node} \\ \textbf{Deg} $\downarrow$} &
    \makecell{\textbf{Edge} \\ \textbf{Size} $\downarrow$} & 
    \textbf{Spectral $\downarrow$} &
    \makecell{\textbf{Node} \\ \textbf{Deg} $\downarrow$} &
    \makecell{\textbf{Edge} \\ \textbf{Size} $\downarrow$} &
    \textbf{Spectral $\downarrow$}\\
    \midrule
    \xmark & \cmark &57.5\%  &1.234  &0.006  &\textbf{0.009}  &77.5\%  &0.861  &0.654  &0.149 &20\%  &0.134 &0.088  &0.014  &1.018  &0.045  &0.009  \\
    \cmark & \xmark &47.5\%  &\textbf{0.148}  &0.005  &0.011 &77.5\%  &0.115  &\textbf{0.146}  &\textbf{0.004} &\textbf{77.5}\%  &0.072  &\textbf{0.015}  &0.026  &1.015  &0.124  &0.014  \\
    \cmark & \cmark &\textbf{65}\%& 0.321& \textbf{0.002}& 0.010& \textbf{90}\%& \textbf{0.063}& 0.220& \textbf{0.004}& \textbf{77.5}\%& \textbf{0.059}& 0.108& \textbf{0.012} &\textbf{0.475}& \textbf{0.012}& \textbf{0.006}\\
    \bottomrule
    \end{tabular}%
    }
    \caption{Ablation studies on the upper bound in the number of hyperedges and the spectrum-preserving coarsening.}
    \label{tab:ablations}
\end{table*}

\noindent \textbf{Metrics}.
Our metrics measure: (i) overall structural similarities like \textit{Node Num} (difference in the number of nodes), \textit{Node Deg} (difference in node degrees), and \textit{Edge Size} (difference in the size of the hyperedges); (ii) topological properties by computing the average difference of the \textit{Spectral} properties.
For datasets with specific structural requirements, \textit{Valid} metrics assess the fraction of generated samples satisfying these properties. 
Lower values indicate better performance for all metrics except \textit{Valid}, where higher is better.
More details can be found in Appendix~\ref{app:metrics}.

\vspace{0.1cm}

\noindent \textbf{Comparison with the Baselines}.
Tables~\ref{tab:hypergraph_SBM_Ego_Tree_results} and \ref{tab:hypergraph_results_er_modelnet40} show the comparison of \method~against our baselines.
We observe that the proposed method effectively captures the edge size distribution and successfully imitates structural properties across datasets.
This is particularly evident in the Ego dataset, where our approach uniquely generates correct ego hypergraphs with a high success rate of 90\%.
In some cases, HyperPA has a similar or better performance compared to \method, but this baseline requires to know a-priori the true node degree and hyperedge distributions on a per-hypergraph basis, whereas our model only requires the desired number of nodes.

The primary advantage of \method~over other baseline approaches lies in its comprehension of hypergraph structure.
This is particularly evident in the \textit{Valid} metrics, where only \method~achieves satisfactory results.
Indeed, node degrees and hyperedge sizes can be replicated by merely outputting the correct density of white pixels on the incidence matrix images.
This explains the satisfactory results sometimes achieved by these baselines for \textit{Node Deg} and \textit{Edge Size}.
However, the underlying structure cannot be captured this way, and our baselines fail the \textit{Valid} metrics.

While one might consider using graph generators to produce hypergraph representations, this approach is infeasible because: (i) generating the clique expansion and recovering the associated hypergraph is an NP-hard problem as it requires the enumeration of all cliques; and (ii) for the bipartite representation, determining which side corresponds to nodes and which to hyperedges is non-trivial.
Additionally, we empirically observe that graph-based models often struggle to generate valid bipartite graphs.
For example, we only obtained 30\% of correct bipartite graphs for both models by \citet{bergmeister2024efficient} and \citet{vignac2023digress}.

\vspace{0.1cm}

\noindent \textbf{Ablation Studies}. Table~\ref{tab:ablations} shows the results of our ablation studies. 
First, we observe that not enforcing an upper limit on the hyperedge cluster sizes makes the hyperedge generation task more difficult: for the four datasets, \textit{Node Deg} greatly increases.
This is especially harmful to the \textit{Valid Tree} and the Ego's \textit{Spectral} metrics as their structure requires very specific sparsity properties.
For the four datasets, the model overestimates the correct number of hyperedges and produces denser hypergraphs.

Then, we also observe that the effects of spectrum-preserving coarsening are more subtle.
SBM and Ego hypergraphs suffer the most from its absence. This is expected since failing to preserve their structure during coarsening reduces the available information the model possesses during generation.
Tree hypergraphs are not heavily impacted due to their relative absence of global structure, whereas ER hypergraphs suffer in the \textit{Node Deg} and \textit{Edge Size} metrics as the model struggles to correctly generate the hyperedges.
This is also expected since the spectrum of the hypergraph contains information on the hyperedge distribution.

\vspace{0.1cm}

\noindent \textbf{Limitations}.
The mesh datasets reveal that \method~faces difficulties in accurately producing the specified number of nodes and hyperedges. The model appears to sample from the underlying distribution of hypergraph sizes rather than adhering to the node count directive provided during the generation process.
Notably, the \textit{Node Num} metric approximates the standard deviation ($std$) of node counts for each dataset.
We hypothesize it stems from an inability to correctly estimate the number of hyperedges. The model then discards the excess nodes by disconnecting them to keep the targeted properties in the remainder of the hypergraph.

\section{Conclusion}
In this work, we introduced \method~which is, to the best of our knowledge, the first attempt at diffusion-based hypergraph generation.
We generalized the iterative local expansion scheme by \citet{bergmeister2024efficient} and the coarsening process by \citet{loukas2018graph} for hypergraphs.
Therefore, we introduced an iterative local expansion procedure for the generation of hyperedges.
We trained a denoising diffusion model and successfully tested the ability of our method to generate hypergraphs sampled from specific distributions. This work provides a key contribution to graph generation, being one of the first able to directly generate hypergraphs.

\section*{Acknowledgments}
The authors acknowledge the ANR – FRANCE (French National Research Agency) for its financial support of the System On Chip Design leveraging Artificial Intelligence (SODA) project under grant ANR-23-IAS3-0004 and the JCJC project DeSNAP ANR-24-CE23-1895-01. This project has also been partially funded by the Hi!PARIS Center on Data Analytics and Artificial Intelligence.

\bibliography{main}

@inproceedings{bergmeister2024efficient,
  title={Efficient and scalable graph generation through iterative local expansion},
  author={Bergmeister, Andreas and Martinkus, Karolis and Perraudin, Nathana{\"e}l and Wattenhofer, Roger},
  booktitle={International Conference on Learning Representations},
  year={2024}
}

@inproceedings{kong2023autoregressive,
  title={Autoregressive diffusion model for graph generation},
  author={Kong, Lingkai and Cui, Jiaming and Sun, Haotian and Zhuang, Yuchen and Prakash, B Aditya and Zhang, Chao},
  booktitle={International Conference on Machine Learning},
  year={2023}
}

@inproceedings{chen2023efficientdegreeguidedgraphgeneration,
  title={Efficient and degree-guided graph generation via discrete diffusion modeling},
  author={Chen, Xiaohui and He, Jiaxing and Han, Xu and Liu, Li-Ping},
  booktitle={International Conference on Machine Learning},
  year={2023}
}

@inproceedings{vignac2023digress,
  title={Digress: Discrete denoising diffusion for graph generation},
  author={Vignac, Clement and Krawczuk, Igor and Siraudin, Antoine and Wang, Bohan and Cevher, Volkan and Frossard, Pascal},
  booktitle={International Conference on Learning Representations},
  year={2023}
}

@inproceedings{guo2023unpooling,
  title={An Unpooling Layer for Graph Generation},
  author={Guo, Yinglong and Zou, Dongmian and Lerman, Gilad},
  booktitle={International Conference on Artificial Intelligence and Statistics},
  year={2023}
}

@article{BOLLA199319,
  title={Spectra, Euclidean representations and clusterings of hypergraphs},
  author={Bolla, Marianna},
  journal={Discrete Mathematics},
  year={1993},
  publisher={Elsevier}
}

@techreport{Zhou2005,
  title = {Beyond pairwise classification and clustering using hypergraphs},
  author = {Zhou, D. and Huang, J. and Sch{\"o}lkopf, B.},
  number = {143},
  organization = {Max-Planck-Gesellschaft},
  institution = {Max Planck Institute for Biological Cybernetics},
  school = {Biologische Kybernetik},
  year = {2005},
  month_numeric = {8}
}

@inproceedings{HigherOrderLearning,
  title={Higher order learning with graphs},
  author={Agarwal, Sameer and Branson, Kristin and Belongie, Serge},
  booktitle={International Conference on Machine Learning},
  year={2006}
}

@article{loukas2018graph,
  title={Graph reduction with spectral and cut guarantees},
  author={Loukas, Andreas},
  journal={Journal of Machine Learning Research},
  year={2019}
}

@inproceedings{pmlr-v97-kajino19a,
  title={Molecular hypergraph grammar with its application to molecular optimization},
  author={Kajino, Hiroshi},
  booktitle={International Conference on Machine Learning},
  year={2019}
}

@article{Higham_2021,
  title={Epidemics on hypergraphs: Spectral thresholds for extinction},
  author={Higham, Desmond J and De Kergorlay, Henry-Louis},
  journal={Proceedings of the Royal Society A},
  year={2021},
  publisher={The Royal Society Publishing}
}

@article{Estrada_2006,
  title={Subgraph centrality and clustering in complex hyper-networks},
  author={Estrada, Ernesto and Rodr{\'\i}guez-Vel{\'a}zquez, Juan A},
  journal={Physica A: Statistical Mechanics and its Applications},
  year={2006},
  publisher={Elsevier}
}

@article{Starostin,
  title={The use of hypergraphs for solving the problem of orthogonal routing of large-scale integrated circuits with an irregular structure},
  author={Starostin, NV and Balashov, VV},
  journal={Journal of Communications Technology and Electronics},
  year={2008},
  publisher={Springer}
}

@article{gong2023generativehypergraphmodelsspectral,
  title={Generative hypergraph models and spectral embedding},
  author={Gong, Xue and Higham, Desmond J and Zygalakis, Konstantinos},
  journal={Scientific Reports},
  year={2023},
  publisher={Nature Publishing Group UK London}
}

@article{10336546,
  title={Automatic hypergraph generation for enhancing recommendation with sparse optimization},
  author={Lin, Zhenghong and Yan, Qishan and Liu, Weiming and Wang, Shiping and Wang, Menghan and Tan, Yanchao and Yang, Carl},
  journal={IEEE Transactions on Multimedia},
  year={2023},
  publisher={IEEE}
}

@inproceedings{arafat2020constructionrandomgenerationhypergraphs,
  title={Construction and random generation of hypergraphs with prescribed degree and dimension sequences},
  author={Arafat, Naheed Anjum and Basu, Debabrota and Decreusefond, Laurent and Bressan, St{\'e}phane},
  booktitle={International Conference on Database and Expert Systems Applications},
  year={2020}
}

@inproceedings{simonovsky2018graphvaegenerationsmallgraphs,
  title={GraphVAE: Towards generation of small graphs using variational autoencoders},
  author={Simonovsky, Martin and Komodakis, Nikos},
  booktitle={International Conference on Artificial Neural Networks},
  year={2018}
}

@inproceedings{you2018graphrnngeneratingrealisticgraphs,
  title={GraphRNN: Generating realistic graphs with deep auto-regressive models},
  author={You, Jiaxuan and Ying, Rex and Ren, Xiang and Hamilton, William and Leskovec, Jure},
  booktitle={International Conference on Machine Learning},
  year={2018}
}

@inproceedings{martinkus2022spectrespectralconditioninghelps,
  title={Spectre: Spectral conditioning helps to overcome the expressivity limits of one-shot graph generators},
  author={Martinkus, Karolis and Loukas, Andreas and Perraudin, Nathana{\"e}l and Wattenhofer, Roger},
  booktitle={International Conference on Machine Learning},
  year={2022}
}

@inproceedings{lim2022signbasisinvariantnetworks,
  title={Sign and basis invariant networks for spectral graph representation learning},
  author={Lim, Derek and Robinson, Joshua and Zhao, Lingxiao and Smidt, Tess and Sra, Suvrit and Maron, Haggai and Jegelka, Stefanie},
  booktitle={International Conference on Learning Representations},
  year={2022}
}

@inproceedings{luo2024dehnneffectiveneuralmodel,
  title={DE-HNN: An effective neural model for circuit netlist representation},
  author={Luo, Zhishang and Hy, Truong Son and Tabaghi, Puoya and Defferrard, Micha{\"e}l and Rezaei, Elahe and Carey, Ryan M and Davis, Rhett and Jain, Rajeev and Wang, Yusu},
  booktitle={International Conference on Artificial Intelligence and Statistics},
  year={2024}
}

@inproceedings{niu2020permutationinvariantgraphgeneration,
  title={Permutation invariant graph generation via score-based generative modeling},
  author={Niu, Chenhao and Song, Yang and Song, Jiaming and Zhao, Shengjia and Grover, Aditya and Ermon, Stefano},
  booktitle={International Conference on Artificial Intelligence and Statistics},
  year={2020}
}

@inproceedings{zhu2023molhfhierarchicalnormalizingflow,
  title={MOLHF: A hierarchical normalizing flow for molecular graph generation},
  author={Zhu, Yiheng and Ouyang, Zhenqiu and Liao, Ben and Wu, Jialu and Wu, Yixuan and Hsieh, Chang-Yu and Hou, Tingjun and Wu, Jian},
  booktitle={International Joint Conference on Artificial Intelligence},
  year={2023}
}

@inproceedings{Do_2020,
  title={Structural patterns and generative models of real-world hypergraphs},
  author={Do, Manh Tuan and Yoon, Se-eun and Hooi, Bryan and Shin, Kijung},
  booktitle={ACM SIGKDD International Conference on Knowledge Discovery \& Data Mining},
  year={2020}
}

@article{aksoy2020hypernetworksciencehighorderhypergraph,
  title={Hypernetwork science via high-order hypergraph walks},
  author={Aksoy, Sinan G and Joslyn, Cliff and Marrero, Carlos Ortiz and Praggastis, Brenda and Purvine, Emilie},
  journal={EPJ Data Science},
  year={2020},
  publisher={Springer Berlin Heidelberg}
}

@article{molecular,
  title={Hypergraph geometry reflects higher-order dynamics in protein interaction networks},
  author={Murgas, Kevin A and Saucan, Emil and Sandhu, Romeil},
  journal={Scientific Reports},
  year={2022},
  publisher={Nature Publishing Group UK London}
}

@article{GRZESIAKKOPEC2017491,
  title={Hypergraphs and extremal optimization in 3D integrated circuit design automation},
  author={Grzesiak-Kope{\'c}, Katarzyna and Oramus, Piotr and Ogorza{\l}ek, Maciej},
  journal={Advanced Engineering Informatics},
  year={2017},
  publisher={Elsevier}
}

@inproceedings{urbanplanning,
  title={Hypergraph formalism for urban form specification},
  author={Dupagne, A and Teller, A},
  booktitle={COST C4 Final Conference, Kiruna},
  year={1998}
}

@inproceedings{molecularbiology,
  title={Reverse engineering molecular hypergraphs},
  author={Rahman, Ahsanur and Poirel, Christopher L and Badger, David J and Murali, TM},
  booktitle={ACM Conference on Bioinformatics, Computational Biology and Biomedicine},
  year={2012}
}

@inproceedings{zhu2022surveydeepgraphgeneration,
  title={A survey on deep graph generation: Methods and applications},
  author={Zhu, Yanqiao and Du, Yuanqi and Wang, Yinkai and Xu, Yichen and Zhang, Jieyu and Liu, Qiang and Wu, Shu},
  booktitle={Learning on Graphs Conference},
  year={2022}
}

@inproceedings{karras2022elucidatingdesignspacediffusionbased,
  title={Elucidating the design space of diffusion-based generative models},
  author={Karras, Tero and Aittala, Miika and Aila, Timo and Laine, Samuli},
  booktitle={Advances in Neural Information Processing Systems},
  year={2022}
}

@inproceedings{maron2020provablypowerfulgraphnetworks,
  title={Provably powerful graph networks},
  author={Maron, Haggai and Ben-Hamu, Heli and Serviansky, Hadar and Lipman, Yaron},
  booktitle={Advances in Neural Information Processing Systems},
  year={2019}
}

@inproceedings{wu20153dshapenetsdeeprepresentation,
  title={3D shapenets: A deep representation for volumetric shapes},
  author={Wu, Zhirong and Song, Shuran and Khosla, Aditya and Yu, Fisher and Zhang, Linguang and Tang, Xiaoou and Xiao, Jianxiong},
  booktitle={IEEE/CVF Conference on Computer Vision and Pattern Recognition},
  year={2015}
}

@inproceedings{kingma2022autoencodingvariationalbayes,
  title={Auto-encoding variational Bayes},
  author={Kingma, Diederik P and Welling, Max},
  booktitle={International Conference on Learning Representations},
  year={2013}
}

@article{goodfellow2014generativeadversarialnetworks,
  title={Generative adversarial networks},
  author={Goodfellow, Ian and Pouget-Abadie, Jean and Mirza, Mehdi and Xu, Bing and Warde-Farley, David and Ozair, Sherjil and Courville, Aaron and Bengio, Yoshua},
  journal={Communications of the ACM},
  year={2020},
  publisher={ACM New York, NY, USA}
}

@inproceedings{ho2020denoisingdiffusionprobabilisticmodels,
  title={Denoising diffusion probabilistic models},
  author={Ho, Jonathan and Jain, Ajay and Abbeel, Pieter},
  booktitle={Advances in Neural Information Processing Systems},
  year={2020}
}

@article{kim2018stochasticblockmodelhypergraphs,
  title={Stochastic block model for hypergraphs: Statistical limits and a semidefinite programming approach},
  author={Kim, Chiheon and Bandeira, Afonso S and Goemans, Michel X},
  journal={arXiv preprint arXiv:1807.02884},
  year={2018}
}

@article{NIEMINEN199935,
  title={Hypertrees},
  author={Nieminen, J and Peltola, M},
  journal={Applied mathematics letters},
  year={1999},
  publisher={Elsevier}
}

@article{ErdösRényi+2006+38+82,
  title={On the evolution of random graphs},
  author={Erdős, Paul and R{\'e}nyi, Alfr{\'e}d},
  journal={Publ. Math. Inst. Hungar. Acad. Sci},
  year={1960}
}

@inproceedings{comrie2021hypergraphegonetworkstemporalevolution,
  title={Hypergraph ego-networks and their temporal evolution},
  author={Comrie, Cazamere and Kleinberg, Jon},
  booktitle={IEEE International Conference on Data Mining},
  year={2021}
}

@article{vishnoi2013lx,
  title={L{x} = b},
  author={Vishnoi, Nisheeth K.},
  journal={Foundations and Trends{\textregistered} in Theoretical Computer Science},
  volume={8},
  number={1--2},
  pages={1--141},
  year={2013},
  issn={1551-305X}
}

\clearpage
\newpage

\appendix

\onecolumn

\section{Proofs} \label{sec::proofs}
\subsection{Laplacians equality} \label{sec::laplacians_equality}
\begin{proposition*}[Adapted from Section 6.4 of \citet{HigherOrderLearning}] 
    For an unweighted hypergraph, Bolla's \textit{unnormalized} Laplacian $\mathbf{L}_H = \mathbf{D}_\mathcal{V} - \mathbf{H} \mathbf{D}_\mathcal{E}^{-1} \mathbf{H}^T$ is equal to the \textit{unnormalized} Laplacian of the associated clique expansion $C$, where each edge $e_{uv}$ is weighted by $\sum_{e \ni u, v ; \text{ } e \in \mathcal{E}} \frac{1}{|e|}$.
\end{proposition*}

\begin{proof}
Let $H = (\mathcal{V}, \mathcal{E})$ be a hypergraph and $C$ its weighted clique expansion. The unnormalized Laplacian of $H$ is given by $\mathbf{L}_H = \mathbf{D}_\mathcal{V} - \mathbf{H} \mathbf{D}_\mathcal{E}^{-1} \mathbf{H}^T$.

For indices $i \neq j$,
\begin{equation}
    \left( \mathbf{L}_H \right)_{ij} = -\sum_{e \ni i,j ; e \in \mathcal{E}} \frac{1}{|e|} = -\mathbf{A}_{ij}
\end{equation}
where $\mathbf{A}$ is the adjacency matrix of the weighted clique expansion. For an index $i$ on the diagonal, denoting $d_i$ the number of hyperedges containing node $i$,
\begin{align}
    \left( \mathbf{L}_H \right)_{ii} &= d_i - \sum_{e \ni i ; e \in \mathcal{E}} \sum_{j \in e ; j \neq i} \frac{1}{|e|} = d_i - \sum_{e \ni i ; e \in \mathcal{E}} \frac{|e| - 1}{|e|} = \sum_{e \ni i ; e \in \mathcal{E}} \left( 1 - \frac{|e| - 1}{|e|} \right) = \sum_{e \ni i ; e \in \mathcal{E}} \frac{1}{|e|} = \mathbf{D}_{ii}
\end{align}
where $\mathbf{D}$ is the degree matrix of the weighted clique expansion. Thus $\mathbf{L}_H = \mathbf{L}_C$.
\end{proof}

\subsection{Upper Bound on Hyperedge Cluster Size} \label{sec::upper_bound_size}
\begin{proposition*}
    In a single merging of two adjacent nodes in the clique representation, at most three hyperedges can be involved in each hyperedge merging in the bipartite representation.
\end{proposition*}

\begin{proof}
    We begin by establishing the following lemma:
    \begin{lemma*} For each node merging, a hyperedge can merge with at most one other hyperedge of the same size.
    \end{lemma*}
    
    Let $n_1$ and $n_2$ be the nodes being merged. Begin by this observation: consider two merging hyperedges $e_1 = (e_{11}, \dots, e_{1k})$ and $e_2 = (e_{21}, \dots, e_{2l})$ (not necessarily of the same size). They must be identical after merging. If both contain $n_1$ and $n_2$, this implies that they are of the same size $k$ and that their other $k-2$ vertices must be the same, as they remain unchanged. Thus, $e_1 = e_2$.
    
    Now, we prove the lemma by contradiction. Assume there exist three distinct hyperedges of the same size $e_1 = (e_{11}, \dots, e_{1k})$, $e_2 = (e_{21}, \dots, e_{2k})$, and $e_3 = (e_{31}, \dots, e_{3k})$ that merge into a cluster when nodes $n_1$ and $n_2$ merge. Only one of $n_1$ or $n_2$ can appear in each hyperedge (otherwise, they would be of different sizes after merging or would violate our initial observation).
        
    For these hyperedges to merge, their $k-1$ other vertices must be identical. With only two possible choices ($n_1$ or $n_2$) for three hyperedges, at least two must be identical, contradicting our assumption of three distinct hyperedges.
    
    \vspace{10pt}
    For the main result, we again use contradiction. Assume four hyperedges $e_1$, $e_2$, $e_3$, and $e_4$ merge into a hyperedge cluster when nodes $n_1$ and $n_2$ merge. By our initial observation, either one hyperedge is of order $k$ and contains both $n_1$ and $n_2$ while the others are of order $k-1$ with only one of the two nodes, or all are of the same order with only one node appearing in each hyperedge.

    In either case, at least three hyperedges of the same order must merge, contradicting our lemma.
\end{proof}

\subsection{Inverting Coarsening through Expansion and Refinement} \label{sec::proof_inversion}
Here we demonstrate that each coarsened bipartite graph (as per Definition \ref{def::coarsening}) can be inverted using specific expansion and refinement steps. Our proof adapts the approach outlined in Appendix A of \cite{bergmeister2024efficient}.

Let $H$ be an arbitrary hypergraph, $C$ its clique representation, and $B = (\mathcal{V}_L, \mathcal{V}_R, \mathcal{E})$ its bipartite representation. Let $\mathcal{P} = (\mathcal{P}_L, \mathcal{P}_R)$ be a partitioning of $\mathcal{V}_L$ and $\mathcal{V}_R$ according to Definition \ref{def::bipartite_coarsening} based on $C$. Furthermore, let $B^c = (\mathcal{V}_L^c, \mathcal{V}_R^c, \mathcal{E}^c) = \bar{B}(B, \mathcal{P})$ denote the coarsened bipartite representation as per Definition \ref{def::bipartite_coarsening}. We will now construct the expansion and refinement vectors that recover the original bipartite graph $B$ from its coarsening $B^c$.

We begin with the expansion by defining vectors $\mathbf{v}_L \in \mathbb{N}^{|\mathcal{P}_L|}$ and $\mathbf{v}_R \in \mathbb{N}^{|\mathcal{P}_R|}$ as follows:
\begin{align} \label{eq::expansion_vectors_setting}
\begin{split}
    \mathbf{v}_L[p] &= |\mathcal{V}_L^{(p)}| \text{ for all } \mathcal{V}_L^{(p)} \in \mathcal{P}_L\\ \mathbf{v}_R[p] &= |\mathcal{V}_R^{(p)}| \text{ for all } \mathcal{V}_R^{(p)} \in \mathcal{P}_R
\end{split}
\end{align}

Let $B^e = (\mathcal{V}_L^e, \mathcal{V}_R^e, \mathcal{E}^e) = \tilde{B}(B^c, \mathbf{v}_L, \mathbf{v}_R)$ represent the expanded graph as defined in Definition \ref{def::expansion}. Notably, the node sets of $B$ and $B^e$ have the same cardinality. Thus, we can establish two bijections $\phi_L : \mathcal{V}_L \rightarrow \mathcal{V}_L^e$ and $\phi_R : \mathcal{V}_R \rightarrow \mathcal{V}_R^e$ between them. Here, the $i$-th node in the $p$-th part of $\mathcal{P}_L$ maps to the corresponding node $v_L^{(p, i)} \in \mathcal{V}_L^e$ in the expanded graph for $\phi_L$, with a similar mapping for $\phi_R$.

This construction ensures that the edge set of $B^e$ is a superset of the original graph $B$'s edge set. To illustrate, consider an arbitrary edge $e_{\{i,j\}} \in \mathcal{E}$. Due to the bipartite nature of the graphs, $v_L^{(i)}$ and $v_R^{(j)}$ must lie in different partitions $\mathcal{V}_L^{(p)}$ and $\mathcal{V}_R^{(q)}$ respectively. An edge in the original representation implies that the parts representing nodes $v_L^{(p)} \in \mathcal{V}_L^c$ and $v_R^{(q)} \in \mathcal{V}_R^c$ are connected in $B^c$. Consequently, when expanding $v_L^{(p)}$ and $v_R^{(q)}$, all $|\mathcal{V}_L^{(p)}|$ nodes associated with $v_L^{(p)}$ connect to all $|\mathcal{V}_R^{(q)}|$ nodes associated with $v_R^{(q)}$ in $B^e$. Specifically, $v_L^{(\phi_L(i))}$ and $v_R^{(\phi_R(j))}$ are connected in $B^e$.

For the refinement step, we define the vector $\mathbf{e} \in \{0, 1\}^{|\mathcal{E}^e|}$ as follows: given an arbitrary ordering of edges in $\mathcal{E}^e$, let $\mathbf{e}_{(i)} \in \mathcal{E}^e$ denote the $i$-th edge in this ordering. We set:
\begin{equation} \label{eq::edge_vector_setting}
\mathbf{e}[i] = 
\begin{cases}
1 & \text{if } e^{\{\phi_L^{-1}(i),\phi_R^{-1}(j)\}} \in \mathcal{E} \\
0 & \text{otherwise}.
\end{cases}
\end{equation}

As per Definition \ref{def::refinement}, the refined graph is then given by $B^r = \bar{B}(B^e, \mathbf{e}) = \bar{B}(\tilde{B}(B^c, \mathbf{v}_L, \mathbf{v}_R), \mathbf{e})$, which is isomorphic to the original bipartite graph $B$.

\subsection{Spectral Equivalence of Hypergraphs and Their Bipartite Representations} \label{sec::spectral_equivalence_bipartite}
\begin{proposition} Let $H$ be a hypergraph and $B$ its bipartite representation. Denote $\mathbf{\mathcal{L}}_H$ and $\mathbf{\mathcal{L}}_B$ their respective \textit{normalized} Laplacian matrix. We have the following equality:
    \begin{equation} 
        Sp(\mathbf{\mathcal{L}}_B) = \left\{ 1 \pm \sqrt{1 - \lambda} \enspace \vert \enspace \lambda \in Sp(\mathbf{\mathcal{L}}_H) \right\} \subset [0, 2]
    \end{equation}
\end{proposition}

\begin{proof}
    Let $H = (\mathcal{V}, \mathcal{E})$ be a hypergraph, and $B = (\mathcal{V}_L, \mathcal{V_R}, \mathcal{E}_B)$ its bipartite representation, with $\mathcal{V}_B = \mathcal{V}_L \cup \mathcal{V}_R$. Following \cite{HigherOrderLearning}, we can express the \textit{normalized} graph Laplacian of $B$ as:
    \begin{align}
        \mathbf{\mathcal{L}}_B &= \mathbf{I} - \mathbf{D}^{-1/2} \mathbf{A} \mathbf{D}^{-1/2} 
        = \begin{bmatrix}
            \mathbf{I} & -\mathbf{D}_v^{-\frac{1}{2}}\mathbf{H}\mathbf{D}_e^{-\frac{1}{2}} \\
            -\mathbf{D}_e^{-\frac{1}{2}}\mathbf{H}^{\top}\mathbf{D}_v^{-\frac{1}{2}} & \mathbf{I}
         \end{bmatrix} 
        = \begin{bmatrix}
                \mathbf{I} & -\mathbf{M} \\
                -\mathbf{M}^T & \mathbf{I}
            \end{bmatrix}
    \end{align}
    where $\mathbf{M} = \mathbf{D}_v^{-\frac{1}{2}}\mathbf{H}\mathbf{D}_e^{-\frac{1}{2}}$.
    The hypergraph \textit{normalized} Laplacian is then given by $\mathbf{\mathcal{L}}_H = \mathbf{I} - \mathbf{M} \mathbf{M}^T$.

    \textbf{\underline{$\subset$:}} Let $\mu$ be an eigenvalue of $\mathbf{\mathcal{L}}_B$ with eigenvector $\mathbf{v} = \left( \begin{array}{c} \mathbf{x} \\ \mathbf{y} \end{array} \right)$. Then $\mathbf{\mathcal{L}}_B \mathbf{v} = \mu \mathbf{v}$. We find that $\mathbf{y} = 0$ if $\mu = 0$, and $\mathbf{y} = \frac{\mathbf{M}^T\mathbf{x}}{1 - \mu}$ otherwise, implying $\mathbf{v}$ is linear in $\mathbf{x}$. We also have $\mathbf{M} \mathbf{M}^T \mathbf{x} = (1 - \mu)^2 \mathbf{x}$, so $1 - (1 - \mu)^2$ is an eigenvalue of $\mathbf{\mathcal{L}}_H$ with the same multiplicity as in $\mathbf{\mathcal{L}}_B$, \ie all eigenvalue $\mu$ of $\mathbf{\mathcal{L}}_B$ can be written as $1 \pm \sqrt{1 - \lambda}$ with $\lambda$ eigenvalue of $\mathbf{\mathcal{L}}_H$.

    \textbf{\underline{$\supset$:}} Let $\lambda$ be an eigenvalue of $\mathbf{\mathcal{L}}_H$ with eigenvector $\mathbf{x}$. Define $\mu = 1 \pm \sqrt{1 - \lambda}$ and $\mathbf{v} = \left( \begin{array}{c} \mathbf{x} \\ \mathbf{y} \end{array} \right)$. If $\lambda = 0$, set $\mathbf{y} = 0$; otherwise, set $\mathbf{y} = \frac{\mathbf{M}^T \mathbf{x}}{1 - \mu}$. In both cases, $\mathbf{\mathcal{L}}_B \mathbf{v} = \mu \mathbf{v}$, showing that $\mu$ is an eigenvalue of $\mathbf{\mathcal{L}}_B$ with the same multiplicity as $\lambda$ in $\mathbf{\mathcal{L}}_H$.
\end{proof}

\section{Complexity Analysis}

Here we analyze the asymptotic complexity of our algorithm, generalizing the approach by \cite{bergmeister2024efficient}.
To generate a hypergraph of $n$ nodes, $m$ hyperedges and $k$ incidences, our algorithm computes an expansion sequence of bipartite graphs ($B^{(L)} = ( \{1\}, \{2\}, \{(1, 2)\} ), B^{(L-1)}, \dots, B^{(0)} = B)$ where the last one is equal to the bipartite representation of the generated hypergraph.
In the following, $n$, $m$ and $k$ will equivalently denote the number of nodes and number of left-side nodes, number of hyperedges and number of right-side nodes, and number of incidences and number of edges of respectively the hypergraph and its associated bipartite representation.

First, in our setting there exists $\epsilon > 0$ such that the number of left-side nodes $n_l$ of $B^{(l)}$ satisifies $n_l \geq (1+\epsilon)n_{l-1}$ for all step $0 \leq l < L$ in the expansion sequence.
One can take $\epsilon = reduction\_frac / (1 - reduction\_frac)$ for example.
Thus we can bound the length of the expansion sequence by $\lceil \log_{1+\epsilon} n \rceil \in \mathcal{O}(\log n)$.

As expansion only increases the number of nodes, we are assured that all $B_l$ have fewer than $n$ left-side nodes and $m$ right-side nodes.
A similar guarantee does not hold for the number of edges.
Indeed, nothing prevents an intermediate bipartite graph of having more than $k$ edges that will be removed in subsequent refinement steps.
Nevertheless, since during training coarsening only reduces the number of incidences, it is reasonable to expect the model to accurately refine the edges at each step instead of accumulating them, having learned the inverse of the coarsening process.
Hence, for the purpose of this analysis, we assume that $k_l \leq k$ for $0 \leq l \leq L$. A similar reasoning allows to assume that $m_l \leq m$ for $0 \leq l \leq L$.

Now we bound the complexity of generating a step in the expansion sequence.
For $l=L$, this requires instantiating a pair of connected nodes and predicting their expansion numbers $\mathbf{v}_L$, which is of complexity $\mathcal{O}(1)$.

For the other steps $0 \leq l < L$, starting from the bipartite representation $B^{(l+1)}$ and the expansion vectors $\mathbf{v}^{(l+1)}_L$ and $\mathbf{v}^{(l+1)}_R$, the algorithm constructs the expansion $\Tilde{B}(B^{(l+1)}, \mathbf{v}_L^{(l+1)}, \mathbf{v}_R^{(l+1)})$, which is of complexity $\mathcal{O}(n + m)$. Next, it samples $\mathbf{v}_L^{(l)}$, $\mathbf{v}_R^{(l)}$ and $\mathbf{e}^{(l)}$ and constructs the refinement $B^{(l)} = B(\Tilde{B}^{(l)}, \mathbf{e}^{(l)})$.
Let $v^L_{\max}$ and $v^R_{\max}$ denote the maximum size of, respectively, left-side clusters and right-side clusters, which are 2 and 3 in our case.
The number of incidences in the expansion $\Tilde{B}^{(l)}$ can then be upper bounded by $k_l \leq k_{l+1} v^L_{\max} v^R_{\max}$.

Sampling $\mathbf{v}_L^{(l)}$, $\mathbf{v}_R^{(l)}$ and $\mathbf{e}^{(l)}$ requires querying a denoising model a fixed number of times.
The underlying complexity of this process then depends on the chosen model architecture.
In our setting, as a bipartite graph does not contain any triangle, the \textit{Local PPGN} \cite{bergmeister2024efficient} achieves a linear complexity in the number of nodes and incidences, \ie $\mathcal{O}(n + m + k)$.

Then, the complexity of obtaining the node embeddings for $B^{(l)}$ can be bounded by $\tilde{\mathcal{O}}(n + m + k)$.
Indeed, this requires computing the $K$ main eigenvalues and eigenvectors of the graph laplacian of $B^{(l+1)}$.
Using the method of \citet{vishnoi2013lx}, this is of complexity $\tilde{\mathcal{O}}\left(Kk_{l+1}\right)$, where $\tilde{\mathcal{O}}$ hides polylogarithmic factors.
Computing the embeddings using \textit{SignNet} is done with complexity $\mathcal{O}\left(K (n_{l+1} + m_{l+1} + k_{l+1}\right))$, and their replication during the expansion is of complexity linear in the number of nodes and hyperedges.
As $K$ is fixed, the final complexity for computing the node embeddings is $\tilde{\mathcal{O}}(n + m + k)$.

Finally, after the expansion sequence is generated, the final bipartite representation is converted into a hypergraph by collapsing all its right-side nodes into hyperedges, which is of complexity $\mathcal{O}(m + k)$.

Combining all of the above, under the stated assumptions, the complexity to generate a hypergraph $H$ of $n$ nodes, $m$ hyperedges and $k$ incidences is of $\tilde{\mathcal{O}}((n + m + k)\log n)$, again hiding polylogarithmic factors.

\section{Implementation Details}
\subsection{Model Architecture} \label{sec::model_architecture}

In our approach, we treat the expansion numbers for the left and right-side nodes and the existence of edges as attributes of the bipartite graph. We employ the EDM denoising diffusion framework \cite{karras2022elucidatingdesignspacediffusionbased} for modeling $p(\mathbf{v}^{(l)}_L, \mathbf{v}^{(l)}_R, \mathbf{e}^{(l)} | \tilde{B}^{(l)})$. In this framework, the target attributes—the expansion numbers and the existence of edges—are noised, and a denoising model is trained to recover their original values.

The architecture of the denoising model is as follows:

\begin{enumerate}
    \item \textbf{Positional Encoding:} We use SignNet \cite{lim2022signbasisinvariantnetworks} to encode the position of each node in the graph and duplicate the encodings based on the expansion numbers.
    
    \item \textbf{Attribute Embedding:} Three linear layers embed the attributes of the bipartite representation: one for the left-side nodes, one for the right-side nodes, one for the edges of the bipartite representation.
    
    \item \textbf{Feature Concatenation:} 
        \begin{itemize}
            \item For both left and right-side nodes: We concatenate feature embeddings, positional encodings, desired reduction fraction, and targeted hypergraph size (\ie the desired number of left-side nodes).
            \item For edges: We concatenate feature embeddings, concatenated positional encodings of the two nodes forming the edge, desired reduction fraction, and targeted hypergraph size.
        \end{itemize}
    
    \item \textbf{Graph Processing:} These three sets of vectors are then used as attributes of the bipartite graph, which is fed into a succession of sparse PPGN layers (see \cite{bergmeister2024efficient}).
    
    \item \textbf{Output Generation:} The attributes of the resulting graph after the final layer are fed to three linear layers to produce the final predictions: one for the left-side nodes, one for the right-side nodes, one for the edges.
\end{enumerate}

\subsection{Additional Tricks}
\label{sec::miscellaneous}

\noindent \textbf{Deterministic Expansion Size}. Our expansion method typically samples two cluster size vectors $\mathbf{v}_L$ and $\mathbf{v}_R$ to enlarge the graph incrementally until $\mathbf{v}_L$ consists entirely of ones. However, this stochastic approach may not consistently yield graphs of a specified size. Following \cite{bergmeister2024efficient}, to address this we use a deterministic expansion strategy by predetermining the target size of the expanded graph. Instead of sampling $\mathbf{v}_L$, we select nodes with the highest expansion probabilities to achieve the desired size. We also use the proposed reduction fraction of \cite{bergmeister2024efficient} as an additional input during training and inference, calculated as one minus the ratio of node counts between the original and expanded graphs.  Note that this only affects $\mathbf{v}_L$, $\mathbf{v}_R$ is sampled without any limitation, playing a similar role as regular edges in the graph case. Appendix \ref{sec::end_to_end} provides further details on this approach.

\vspace{0.1cm}

\noindent \textbf{Perturbed Expansion}. Again following \citet{bergmeister2024efficient}, while Definitions \ref{def::expansion} and \ref{def::refinement} are enough to reverse coarsening steps, we introduce additional randomness in the expansion process to improve generative performance, especially for limited datasets prone to overfitting. Our perturbed expansion randomly adds with a given probability edges between nodes of both sides of the bipartite representation that are within a predefined radius, supplementing the original edges.

\begin{figure}[htbp]
    \centering
    \includegraphics[width=0.25\columnwidth]{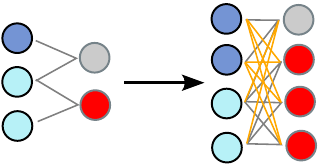}
    \caption{Depiction of a perturbed expansion. The bipartite representation $B^{(l)}$ is expanded into $\tilde{B}^{(l-1)}$ using the cluster size vectors. Deterministic expansion components are represented by gray edges, whereas additional orange edges are added for a radius of $r = 2$ and a probability of $p = 1$. With $p < 1$, a subset of these edges would be randomly excluded.}
    \label{fig:perturbed_expansion}
\end{figure}

The following definition formalizes the concept of a perturbed hypergraph expansion, which is a generalization of the deterministic hypergraph expansion introduced in Definition \ref{def::expansion}. A visual representation of this concept is provided in Figure \ref{fig:perturbed_expansion}.

\begin{definition}[Perturbed Hypergraph Expansion] \label{def::perturbed_expansion}
Given a bipartite representation $B = (\mathcal{V}_L, \mathcal{V}_R, \mathcal{E})$, two cluster size vectors $\mathbf{v}_L \in \mathbb{N}^{|\mathcal{V}_L|}$ and $\mathbf{v}_R \in \mathbb{N}^{|\mathcal{V}_R|}$, a radius $r \in \mathbb{N}$, and a probability $0 \leq p \leq 1$, the perturbed expansion $\tilde{B}$ is constructed as in Definition \ref{def::expansion}, and additionally for all distinct nodes $\mathbf{v}_L{(p)} \in \tilde{\mathcal{V}_L}, \mathbf{v}_R{(q)} \in \tilde{\mathcal{V}}_R$ whose distance in $B$ is at most $2r+1$, we add each edge $e_{\{p_i,q_j\}}$ independently to $\tilde{\mathcal{E}}$ with probability $p$.
\end{definition}

\vspace{0.1cm}

\noindent \textbf{Spectral Conditioning}. Following \citet{martinkus2022spectrespectralconditioninghelps} and \citet{bergmeister2024efficient}, we use principal \textit{normalized} Laplacian eigenvalues and eigenvectors of the target graph as conditional information during generation, which is known to improve generation quality. When generating $B^{(l)}$ from its coarser version $B^{(l+1)}$, we exploit the spectral preservation during coarsening to approximate $B^{(l)}$'s normalized Laplacian spectrum: we compute the $k$ smallest non-zero eigenvalues and corresponding eigenvectors of $B^{(l+1)}$'s \textit{normalized} Laplacian matrix $\mathbf{\mathcal{L}}^{(l+1)}$. Using SignNet (see \cite{lim2022signbasisinvariantnetworks}), we derive node embeddings for $B^{(l+1)}$, which are then replicated across expansion sets to initialize the embeddings of $B^{(l)}$, facilitating cluster identification. The value of $k$ is chosen as a hyperparameter.

\section{Experimental Details}
\subsection{Datasets}
\label{app:datasets}

Our experimental evaluation employs four synthetic and three real-world datasets:
\begin{itemize}
    \item \textbf{Erdos-Renyi hypergraphs} \cite{ErdösRényi+2006+38+82}\textbf{:} These consist of 32 nodes, with 2-edges selected with 0.1 probability, 3-edges with 0.005 probability, and 4-edges with 0.0005 probability.
    \item \textbf{Stochastic Block Model hypergraphs} \cite{kim2018stochasticblockmodelhypergraphs}\textbf{:} Composed of 32 nodes evenly distributed between two groups. All hyperedges are 3-edges, with inter-group edges selected with 0.001 probability and intra-group edges with 0.05 probability.
    \item \textbf{Ego hypergraphs}  \cite{comrie2021hypergraphegonetworkstemporalevolution}\textbf{:} Generated by first creating a hypergraph of 150-200 nodes with 3000 random hyperedges (up to 5 nodes each). A random node is then selected, and the ego hypergraph is constructed by retaining only the hyperedges containing this node.
    \item \textbf{Tree hypergraphs} \cite{NIEMINEN199935}\textbf{:} We construct a tree-structured hypergraph by first generating a tree with 32 nodes using \textit{networkx}. Subsequently, we iteratively combine adjacent edges to form hyperedges, with each hyperedge encompassing up to 5 nodes.
    \item \textbf{ModelNet40 meshes} \cite{wu20153dshapenetsdeeprepresentation}\textbf{:} We convert mesh topologies from selected classes of the ModelNet40 dataset into hypergraphs. To manage computational complexity, we create low-poly versions with fewer than 1000 vertices. The classes used are \textit{bookshelf}, \textit{piano}, and \textit{plant}. Low-poly versions are created by iteratively merging vertices closer than a threshold until the desired vertex count is achieved. Duplicate triangles are then consolidated.
\end{itemize}
All datasets are randomly partitioned into 128 hypergraphs for training, 32 for validation, and 40 for testing.

\subsection{Evaluation Metrics}
\label{app:metrics}

We assess the generated hypergraphs against the test dataset using the following metrics:
\begin{itemize}
    \item \textbf{NodeNumDiff:} Average difference in node count between target and generated hypergraphs.
    \item \textbf{NodeDegreeDistrWasserstein:} Wasserstein distance between node degree distributions of test and generated hypergraphs.
    \item \textbf{EdgeSizeDistrWasserstein:} Wasserstein distance between edge size distributions of test and generated hypergraphs.
    \item \textbf{Spectral:} Maximum Mean Discrepancy between Laplacian spectra of test and generated hypergraphs.
    \item \textbf{Uniqueness:} Proportion of non-isomorphic generated hypergraphs.
    \item \textbf{Novelty:} Proportion of generated hypergraphs that are non-isomorphic to a sample from the training set.
    \item \textbf{CentralityCloseness, CentralityBetweenness, CentralityHarmonic:} Wasserstein distances between distributions of each centrality measure for test and generated hypergraphs, computed on edges for $s = 1$. For details, see \cite{aksoy2020hypernetworksciencehighorderhypergraph}.
    \item \textbf{ValidEgo:} Applicable only to the \textit{hypergraphEgo} dataset, measuring the proportion of generated hypergraphs that are valid ego hypergraphs (\ie containing an ego node present in all hyperedges).
    \item \textbf{ValidSBM:} Applicable only to the \textit{hypergraphSBM} dataset, measuring the proportion of generated hypergraphs that are valid SBM hypergraphs (\ie containing two clusters with approximately the same probability of intra and inter-cluster hyperedges as the training set).
    \item \textbf{ValidTree:} Applicable only to the \textit{hypergraphTree} dataset, measuring the proportion of generated hypergraphs that are valid tree hypergraphs.
\end{itemize}

\subsection{Baseline Methods}
We evaluate our method against the following baseline approaches:
\begin{itemize}
    \item \textbf{HyperPA} \cite{Do_2020}: An algorithmic method for generating hypergraphs.
    \item \textbf{Incidence Matrix Image-based Models}: We implement three simple baselines --- Diffusion, GAN, and VAE --- that generate hypergraphs via incidence matrix images:
    \begin{itemize}
        \item These models are trained to produce images representing incidence matrices, where white pixels denote the presence of a node in a hyperedge, and black pixels indicate absence.
        \item We randomly permute rows and columns, then pad with black pixels to maintain consistent dimensions.
        \item A thresholding operation is applied to the generated images to obtain the final incidence matrices.
    \end{itemize}
\end{itemize}

\section{Coarsening Sequence Sampling} \label{sec::coarsening}
This section outlines our methodology for sampling a coarsening sequence $\pi \in \Pi_F(H)$ for a given hypergraph $H$. Algorithm \ref{alg::hypergraph_coarsening} presents our approach in detail.

Consider a coarsening step $l$ where $H^{(l)}$ denotes the coarsened hypergraph, $B^{(l)}$ its bipartite representation, and $C^{(l)}$ its weighted clique expansion. The process begins by sampling a reduction fraction $red\_frac$ from the interval $[\rho_\text{min}, \rho_\text{max}]$. We then compute the cost of all contraction sets $F(C^{(l)})$ for the weighted clique expansion, with lower costs indicating higher contraction preference.

We use a greedy and randomized strategy, inspecting contraction sets from lowest to highest cost. For each set:

\begin{itemize}
    \item The set is rejected with probability $\lambda$.
    \item If accepted:
    \begin{itemize}
        \item \textbf{First accepted contraction:} We compute the coarsened weighted clique expansion (Definition \ref{def::coarsening}) and the coarsened bipartite representation (Definition \ref{def::bipartite_coarsening}). The contraction is added to the set of applied contractions.

        \item \textbf{Subsequent accepted contractions:} If any node in the contraction set already belongs to an accepted contraction, we reject it. Otherwise, we compute the coarsened representations based on previously accepted contractions. If all right-side clusters in the bipartite representation comprise at most three nodes of $B^{(l)}$, we add the contraction to the set of applied contractions. If not, we reject it and revert to the previous coarsened representation.
    \end{itemize}
\end{itemize}

The process terminates when $|\mathcal{V}_L| - |\bar{\mathcal{V}_L}| > red\_frac * |\mathcal{V}_L|$, \ie when the number of nodes in the hypergraph has been reduced by $red\_frac$. Note that only the left-side (corresponding to the hypergraph nodes) is considered in this stopping criterion.

This approach is flexible, allowing for various choices of cost function $c$, contraction family $F$, reduction fraction range $[\rho_\text{min}, \rho_\text{max}]$, and randomization parameter $\lambda$.

\vspace{0.5cm}

\noindent \textbf{Practical Considerations}. To address the potential imbalance caused by an abundance of small graphs in the coarsening sequence, we adopt a strategy similar to \cite{bergmeister2024efficient}. Specifically, when the current graph has fewer than 16 nodes, we automatically set the reduction fraction $\rho$ to its maximum value, $\rho_\text{max}$. Due to the constraint that for any coarsening step, no right-side cluster can contain more than three nodes, it is not always feasible to achieve a partitioning that achieves the desired reduction fraction. Empirically, provided that $\rho_\text{max}$ is small enough, this is rarely the case in practice.

During the training phase, our approach involves sampling a coarsening sequence from each dataset graph. However, we only use a graph from a randomly selected level within this sequence. As a result, the practical implementation of Algorithm 1 is designed to return a coarsened graph with its associated node and edge features, rather than the entire coarsening sequence $\pi$.
To enhance computational efficiency, we kept the caching mechanism of \cite{bergmeister2024efficient}. Upon generating a coarsening sequence, we store its components in a cache. During training, we then randomly select a level, return the corresponding graph and features, and remove this element from the cache. This approach allows us to avoid unnecessary recomputation, as we only need to regenerate the coarsening sequence for a specific graph when all cached elements have been exhausted.

\vspace{1.5cm}

\noindent \textbf{Hyperparameters}. 
In all experiments described in Section \ref{sec::experiments}, we use the following settings:
\begin{itemize}
    \item Contraction family: The set of all edges in the clique representation, \ie $F(C) = \mathcal{E}$, for a weighted clique expansion $C = (\mathcal{V}, \mathcal{E})$.
    \item Cost function: Local Variation Cost \cite{loukas2018graph} with a preserving eigenspace size of $k = 8$.
    \item Reduction fraction range: $[\rho_\text{min}, \rho_\text{max}] = [0.1, 0.3]$.
    \item Randomization parameter: $\lambda = 0.3$.
\end{itemize}

\begin{algorithm*}[hbtp]
\caption{\textbf{Hypergraph coarsening sequence sampling:} This describes the randomized iterative coarsening of a hypergraph. At each step, contraction sets are selected based on their cost and a sampled reduction fraction, while ensuring that right-side clusters contain at most three nodes. The process continues until the bipartite representation is reduced to a single pair of connected nodes.}
\label{alg::hypergraph_coarsening}
\begin{algorithmic}[1]
\Parameters contraction family $F$, cost function $c$, reduction fraction range $[\rho_\text{min}, \rho_\text{max}]$, randomization parameter $\lambda$
\Input hypergraph $H$
\Output coarsening sequence $\pi = (B^{(0)}, \ldots, B^{(L)})$
\Function{HypergraphCoarseningSeq}{$H$}
    \State $H^{(0)} \gets H$, $B^{(0)} \gets \text{BipartiteRepresentation}(H^{(0)})$, $C^{(0)} \gets \text{WeightedCliqueExpansion}(H^{(0)})$
    \State $\pi \gets (B^{(0)})$
    \State $l \gets 0$

    \While{$|\mathcal{V}_L^{(l)}| > 1$}
        \State $l \gets l + 1$
        \State $\text{red\_frac} \sim \text{Uniform}([\rho_\text{min}, \rho_\text{max}])$ \Comment{random reduction fraction}
        \State $f \gets c(\cdot, C^{(l-1)})$ \Comment{cost function for clique expansion}
        \State $\text{accepted\_contractions} \gets \emptyset$
        \State $C_\text{next} \gets C^{(l-1)}$
        \State $B_\text{next} \gets B^{(l-1)}$

        \For{$S \in \text{SortedByCost}(F(C^{(l-1)}), f)$}
            \If{$\text{Random}() > \lambda$}
                \If{$S \cap (\bigcup_{P \in \text{accepted\_contractions}} P) = \emptyset$}
                    \State $C_\text{temp} \gets \text{CoarsenCliqueExpansion}(C_\text{next}, S)$
                    \State $B_\text{temp} \gets \text{CoarsenBipartite}(B_\text{next}, S)$

                    \If{$\forall \text{ right cluster } R \in B_\text{temp}: |R| \leq 3$}
                        \State $\text{accepted\_contractions} \gets \text{accepted\_contractions} \cup \{S\}$
                        \State $C_\text{next} \gets C_\text{temp}$
                        \State $B_\text{next} \gets B_\text{temp}$
                    \EndIf
                \EndIf
            \EndIf

            \If{$|\mathcal{V}_L^{(l-1)}| - |\mathcal{V}_{L,\text{next}}| > \text{red\_frac} \cdot |\mathcal{V}_L^{(l-1)}|$}
                \State \textbf{break}
            \EndIf
        \EndFor

        \State $C^{(l)} \gets C_\text{next}$, $B^{(l)} \gets B_\text{next}$
        \State $\pi \gets \pi \cup \{B^{(l)}\}$
    \EndWhile

    \State \Return $\pi$
\EndFunction
\end{algorithmic}
\end{algorithm*}

\newpage
\clearpage
\section{End-to-end Training and Sampling} \label{sec::end_to_end}

In this section, we detail the end-to-end training and sampling procedures in Algorithms \ref{alg::train} and \ref{alg::sample_deterministic}, respectively. Both assume the deterministic expansion size setting introduced in Section \ref{sec::miscellaneous}. When deterministic expansion size is disabled, the expansion sizes are sampled directly from the model predictions rather than being selected to match a prescribed target expansion size; the corresponding sampling procedure is given in Algorithm \ref{alg::sample}. All algorithms rely on the node embedding computation defined in Algorithm \ref{alg::embeddings}.

\begin{algorithm*}[hbtp]
\caption{\textbf{Node embedding computation:} This describes how node embeddings are computed for a given bipartite representation of a hypergraph.}
\label{alg::embeddings}
\begin{algorithmic}[1]

\Parameters number of spectral features $k$
\Input bipartite representation $B = (\mathcal{V}_L, \mathcal{V}_R, \mathcal{E})$, spectral feature model SignNet$_\theta$, cluster size vectors $\mathbf{v}_L$ and $\mathbf{v}_R$
\Output node embeddings computed for all nodes in $\mathcal{V}_L$ and $\mathcal{V}_R$ and replicated according to $\mathbf{v}_L$ and $\mathbf{v}_R$

\Function{Embeddings}{$B = (\mathcal{V}_L, \mathcal{V}_R, \mathcal{E})$, SignNet$_\theta$, $\mathbf{v}_L$, $\mathbf{v}_R$}

    \State $n_B \gets |\mathcal{V}_L| + |\mathcal{V}_R|$

    \If{$k = 0$}
        \State $\mathbf{H} = [h^{(1)}, \ldots, h^{(n_B)}], \quad
        h^{(i)} \overset{i.i.d.}{\sim} \mathcal{N}(0, I)$
        \Comment{Sample random embeddings}
    \Else
        \State $k_{\mathrm{eff}} \gets \min(k, n_B - 2)$

        \If{$k_{\mathrm{eff}} > 0$}
            \State $[\lambda_1, \ldots, \lambda_{k_{\mathrm{eff}}}],
            [u_1, \ldots, u_{k_{\mathrm{eff}}}]
            \gets \text{EIG}(B)$
            \Comment{Compute smallest non-trivial spectral features}
        \EndIf

        \If{$k_{\mathrm{eff}} < k$}
            \State $[\lambda_{k_{\mathrm{eff}}+1}, \ldots, \lambda_k],
            [u_{k_{\mathrm{eff}}+1}, \ldots, u_k]
            \gets [0, \ldots, 0], [0, \ldots, 0]$
            \Comment{Pad with zeros}
        \EndIf

        \State $\mathbf{H} =
        [h^{(1)}, \ldots, h^{(n_B)}]
        \gets
        \text{SignNet}_\theta(
        [\lambda_1, \ldots, \lambda_k],
        [u_1, \ldots, u_k], B)$
    \EndIf

    \State $\tilde{B} =
    (\mathcal{V}_L^{(1)} \cup \cdots \cup
    \mathcal{V}_L^{(|\mathcal{V}_L|)},
    \mathcal{V}_R^{(1)} \cup \cdots \cup
    \mathcal{V}_R^{(|\mathcal{V}_R|)},
    \tilde{\mathcal{E}})
    \gets
    \tilde{B}(B, \mathbf{v}_L, \mathbf{v}_R)$
    \Comment{Expand as per Definition \ref{def::expansion}}

    \State Replicate embeddings according to cluster sizes:
    \[
        \tilde{\mathbf{H}}[v_L^{(p,i)}]
        =
        \mathbf{H}[v_L^{(p)}],
        \quad
        \forall p \in [|\mathcal{V}_L|],\;
        1 \leq i \leq \mathbf{v}_L[p]
    \]
    \[
        \tilde{\mathbf{H}}[v_R^{(p,i)}]
        =
        \mathbf{H}[v_R^{(p)}],
        \quad
        \forall p \in [|\mathcal{V}_R|],\;
        1 \leq i \leq \mathbf{v}_R[p]
    \]

    \State \Return $\tilde{\mathbf{H}}$

\EndFunction
\end{algorithmic}
\end{algorithm*}

\begin{algorithm*}[htbp]
\caption{\textbf{End-to-end training procedure:} This describes the entire training procedure for our
model.}
\label{alg::train}
\begin{algorithmic}[1]

\Parameters number of spectral features $k$ for node embeddings
\Input dataset $\mathcal{D} = \{H_1, \ldots, H_N\}$, denoising model GNN$_\theta$, spectral feature model SignNet$_\theta$
\Output trained model parameters $\theta$

\Function{Train}{$\mathcal{D}$, GNN$_\theta$, SignNet$_\theta$}

    \While{not converged}
        \State $H \sim \text{Uniform}(\mathcal{D})$
        \Comment{Sample graph}

        \State $(B^{(0)}, \ldots, B^{(L)}) \gets \text{RndRedSeq}(H)$
        \Comment{Sample coarsening sequence by Algorithm \ref{alg::hypergraph_coarsening}}

        \State $l \sim \text{Uniform}(\{0, \ldots, L\})$
        \Comment{Sample level}

        \If{$l = 0$}
            \State $\mathbf{v}_L^{(0)} \gets 1$, $\mathbf{v}_R^{(0)} \gets 1$
        \Else
            \State Set $\mathbf{v}_L^{(l)}$ and $\mathbf{v}_R^{(l)}$
            as in \eqref{eq::expansion_vectors_setting}, s.t. the node sets of $\tilde{B}(B^{(l)}, \mathbf{v}_L^{(l)}, \mathbf{v}_R^{(l)})$
            equal those of $B^{(l-1)}$
        \EndIf

        \If{$l = L$}
            \State $B^{(L+1)} \gets (\{1\}, \{2\}, \{(1,2)\})$
            \State $\mathbf{v}_L^{(L+1)} \gets 1$
            \State $\mathbf{v}_R^{(L+1)} \gets 1$
        \Else
            \State Set $\mathbf{v}_L^{(l+1)}$ and $\mathbf{v}_R^{(l+1)}$
            as in \eqref{eq::expansion_vectors_setting}, s.t. the node sets of $\tilde{B}(B^{(l+1)},
                \mathbf{v}_L^{(l+1)},
                \mathbf{v}_R^{(l+1)})
            $
            equal those of $B^{(l)}$
        \EndIf

        \State Set $\mathbf{e}^{(l)}$ as in Eq. \eqref{eq::edge_vector_setting}, s.t.
        $
            B\left(
            \tilde{B}(B^{(l+1)},
            \mathbf{v}_L^{(l+1)},
            \mathbf{v}_R^{(l+1)}),
            \mathbf{e}^{(l)}
            \right)
            =
            B^{(l)}
        $

        \State $\mathbf{H}^{(l)}
        \gets
        \text{Embeddings}(B^{(l+1)},
        \text{SignNet}_\theta,
        \mathbf{v}_L^{(l+1)},
        \mathbf{v}_R^{(l+1)})$
        \Comment{Compute node embeddings}

        \State $\hat{\rho}^{(l)}
        \gets
        1 -
        \dfrac{|\mathcal{V}_L^{(l+1)}|}
              {|\mathcal{V}_L^{(l)}|}$

        \State $D_\theta
        \gets
        \text{GNN}_\theta(
        \cdot,
        \cdot,
        \tilde{B}^{(l)},
        \mathbf{H}^{(l)},
        |\mathcal{V}_L^{(0)}|,
        \hat{\rho}^{(l)})$

        \State Take gradient descent step on 
            $\nabla_\theta
            \text{DiffusionLoss}
            (\mathbf{v}_L^{(l)},
            \mathbf{v}_R^{(l)},
            \mathbf{e}^{(l)},
            D_\theta)
        $

    \EndWhile

    \State \Return $\theta$

\EndFunction
\end{algorithmic}
\end{algorithm*}

\begin{algorithm*}[htbp]
\caption{\textbf{End-to-end sampling procedure with deterministic expansion:} This describes the sampling procedure with deterministic expansion, described in Section \ref{sec::miscellaneous}. Note that this assumes that maximum cluster sizes are 2 for nodes and 3 for hyperedges, which is the case when using edges of the clique representation as the contraction set family.}
\label{alg::sample_deterministic}
\begin{algorithmic}[1]
\Parameters reduction fraction range $[\rho_\text{min}, \rho_\text{max}]$
\Input target hypergraph size $N$, denoising model GNN$_\theta$, spectral feature model SignNet$_\theta$
\Output sampled hypergraph $H = (\mathcal{V}, \mathcal{E})$ with $|\mathcal{V}| = N$
\Function{Sample}{$N$, GNN$_\theta$, SignNet$_\theta$}
    \State $B = (\mathcal{V}_L, \mathcal{V}_R, \mathcal{E}) \gets (\{1\}, \{2\}, \{(1, 2)\})$ \Comment{Start with a minimal bipartite graph}
    \State $\mathbf{v}_L \gets [1]$, $\mathbf{v}_R \gets [1]$ \Comment{Initial cluster size vectors}
    \While{$|\mathcal{V}_L| < N$}
        \State $\mathbf{H} \gets \text{Embeddings}(B, \text{SignNet}_\theta, \mathbf{v}_L, \mathbf{v}_R)$ \Comment{Compute node embeddings}
        \State $\tilde{B} \gets \tilde{B}(B, \mathbf{v}_L, \mathbf{v}_R)$
        \State $n \gets \|\mathbf{v}_L\|_1$
        \State $\rho \sim \text{Uniform}([\rho_\text{min}, \rho_\text{max}])$ \Comment{random reduction fraction}
        \State
    \( n_{\text{next}} \gets
    \min\left(
        \max\left(
            \left\lceil \frac{n}{1-\rho} \right\rceil,
            n+1
        \right),
        N
    \right),
    \qquad
    n^+ \gets n_{\text{next}} - n \)
    \Comment{number of left-side nodes to add}
        \State $\hat{\rho} \gets 1 - \dfrac{n}{n_{\text{next}}}$ \Comment{actual reduction fraction}
        \State $D_\theta \gets \text{GNN}_\theta(\cdot, \cdot, \tilde{B}, \mathbf{H}, N, \hat{\rho})$
        \State $(\mathbf{v}_L)_0, (\mathbf{v}_R)_0, (\mathbf{e})_0 \gets \text{Sample}(D_\theta)$ \Comment{Sample features}

        \State Set $\mathbf{v}_L$ s.t. for $i \in [n]$:
        \[
            \mathbf{v}_L[i] =
            \begin{cases}
                2 & \text{if }
                \left|
                    \left\{
                        j \in [n]
                        \mid
                        (\mathbf{v}_L)_0[j] \geq (\mathbf{v}_L)_0[i]
                    \right\}
                \right|
                \leq n^+,
                \\
                1 & \text{otherwise}
            \end{cases}
        \]

        \State Set $\mathbf{v}_R$ s.t. for $i \in [|(\mathbf{v}_R)_0|]$:
        \[
            \mathbf{v}_R[i] =
            \begin{cases}
                1 & \text{if } (\mathbf{v}_R)_0[i] < 1.66,\\
                2 & \text{if } 1.66 \leq (\mathbf{v}_R)_0[i] < 2.33,\\
                3 & \text{otherwise}
            \end{cases}
        \]

        \State Set $\mathbf{e}$ s.t. for $i \in [|(\mathbf{e})_0|]$:
        \[
            \mathbf{e}[i] =
            \begin{cases}
                1 & \text{if } (\mathbf{e})_0[i] > 0.5,\\
                0 & \text{otherwise}
            \end{cases}
        \]

        \State $B = (\mathcal{V}_L, \mathcal{V}_R, \mathcal{E}) \gets B(\tilde{B}, \mathbf{e})$ \Comment{Refine as per Definition \ref{def::refinement}}
    \EndWhile
    \State Build $H$ from its bipartite representation $B$
    \State \Return $H$
\EndFunction
\end{algorithmic}
\end{algorithm*}

\begin{algorithm*}[htbp]
\caption{\textbf{End-to-end sampling procedure without deterministic expansion:} This describes the entire sampling procedure without deterministic expansion. Note that this assumes that maximum cluster sizes are 2 for nodes and 3 for hyperedges, which is the case when using edges of the clique representation as the contraction set family.}
\label{alg::sample}
\begin{algorithmic}[1]

\Parameters reduction fraction range $[\rho_\text{min}, \rho_\text{max}]$
\Input target hypergraph size $N$, denoising model GNN$_\theta$, spectral feature model SignNet$_\theta$
\Output sampled hypergraph $H = (\mathcal{V}, \mathcal{E})$

\Function{Sample}{$N$, GNN$_\theta$, SignNet$_\theta$}

    \State $B = (\mathcal{V}_L, \mathcal{V}_R, \mathcal{E}) 
    \gets (\{1\}, \{2\}, \{(1, 2)\})$ 
    \Comment{Start with a minimal bipartite graph}

    \State $\mathbf{v}_L \gets [1]$, $\mathbf{v}_R \gets [1]$ 
    \Comment{Initial cluster size vectors}

    \While{$|\mathcal{V}_L| < N$}

        \State $\mathbf{H} 
        \gets \text{Embeddings}(B, \text{SignNet}_\theta, \mathbf{v}_L, \mathbf{v}_R)$ 
        \Comment{Compute node embeddings}

        \State $\tilde{B} 
        \gets \tilde{B}(B, \mathbf{v}_L, \mathbf{v}_R)$

        \State $n \gets \|\mathbf{v}_L\|_1$

        \State $\rho \sim \text{Uniform}([\rho_\text{min}, \rho_\text{max}])$ 
        \Comment{random reduction fraction}

        \State
        $
            n_{\text{next}} \gets
            \min\left(
                \max\left(
                    \left\lceil \frac{n}{1-\rho} \right\rceil,
                    n+1
                \right),
                N
            \right)
        $

        \State $\hat{\rho} 
        \gets 1 - \dfrac{n}{n_{\text{next}}}$ 
        \Comment{conditioning reduction fraction}

        \State $D_\theta 
        \gets \text{GNN}_\theta(
        \cdot, \cdot, \tilde{B}, \mathbf{H}, N, \hat{\rho})$

        \State $(\mathbf{v}_L)_0, (\mathbf{v}_R)_0, (\mathbf{e})_0 
        \gets \text{SDESample}(D_\theta)$ 
        \Comment{Sample feature embeddings}

        \State et $\mathbf{v}_L$ s.t. for $i \in [|(\mathbf{v}_L)_0|]$:
        \[
            \mathbf{v}_L[i] =
            \begin{cases}
                1 & \text{if } (\mathbf{v}_L)_0[i] < 1.5,\\
                2 & \text{otherwise}
            \end{cases}
        \]

        \State Set $\mathbf{v}_R$ s.t. for $i \in [|(\mathbf{v}_R)_0|]$:
        \[
            \mathbf{v}_R[i] =
            \begin{cases}
                1 & \text{if } (\mathbf{v}_R)_0[i] < 1.66,\\
                2 & \text{if } 1.66 \leq (\mathbf{v}_R)_0[i] < 2.33,\\
                3 & \text{otherwise}
            \end{cases}
        \]

        \State Set $\mathbf{e}$ s.t. for $i \in [|(\mathbf{e})_0|]$:
        \[
            \mathbf{e}[i] =
            \begin{cases}
                1 & \text{if } (\mathbf{e})_0[i] > 0.5,\\
                0 & \text{otherwise}
            \end{cases}
        \]

        \State $B = (\mathcal{V}_L, \mathcal{V}_R, \mathcal{E}) 
        \gets B(\tilde{B}, \mathbf{e})$ 
        \Comment{Refine as per Definition \ref{def::refinement}}

    \EndWhile

    \State Build $H$ from its bipartite representation $B$
    \State \Return $H$

\EndFunction
\end{algorithmic}
\end{algorithm*}

\newpage
\clearpage

\section{Detailed Results} \label{sec::detailed_results}

\begin{table*}[h!]
    \centering
    \setlength{\tabcolsep}{3pt}  
    \resizebox{\textwidth}{!}{%
    \begin{tabular}{@{}l*{9}{c}*{10}{c}@{}}
    \toprule
    \multirow{3}{*}{\textbf{Model}} & \multicolumn{9}{c}{\textbf{Erdos-Renyi Hypergraphs} ($n_{avg} = $ 32, $std = $ 0.07)} & \multicolumn{10}{c}{\textbf{SBM Hypergraphs} ($n_{avg} = $ 31.73, $std = $ 0.55)} \\
    \cmidrule(l){2-10}
    \cmidrule(l){11-20}
    &
    \makecell{\textbf{Node} \\ \textbf{Num} $\downarrow$} &
    \makecell{\textbf{Node} \\ \textbf{Deg} $\downarrow$} &
    \makecell{\textbf{Edge} \\ \textbf{Size} $\downarrow$} &
    \makecell{\textbf{Spec-} \\ \textbf{tral} $\downarrow$} &
    \makecell{\textbf{Uniq.} $\uparrow$}&
    \makecell{\textbf{Nov.} $\uparrow$}&
    \makecell{\textbf{Cent.} \\ \textbf{Close} $\downarrow$} &
    \makecell{\textbf{Cent.} \\ \textbf{Betw.} $\downarrow$} &
    \makecell{\textbf{Cent.} \\ \textbf{Harm.} $\downarrow$} &
    \makecell{\textbf{Valid} \\ \textbf{SBM} $\uparrow$} &
    \makecell{\textbf{Node} \\ \textbf{Num} $\downarrow$} &
    \makecell{\textbf{Node} \\ \textbf{Deg} $\downarrow$} &
    \makecell{\textbf{Edge} \\ \textbf{Size} $\downarrow$} &
    \makecell{\textbf{Spec-} \\ \textbf{tral} $\downarrow$} &
    \makecell{\textbf{Uniq.} $\uparrow$}&
    \makecell{\textbf{Nov.} $\uparrow$}&
    \makecell{\textbf{Cent.} \\ \textbf{Close} $\downarrow$} &
    \makecell{\textbf{Cent.} \\ \textbf{Betw.} $\downarrow$} &
    \makecell{\textbf{Cent.} \\ \textbf{Harm.} $\downarrow$} \\
    \midrule
    HyperPA &\textbf{0.000} &5.530 &0.183 &0.177 &\textbf{1} &\textbf{1} &0.078 &0.014 &107.1 &2.5\% &\textbf{0.075} &4.062 &0.407 &0.273 &\textbf{1} &\textbf{1} &0.074 &0.008 &77.84 \\
    VAE &0.100 &2.140 &0.540 &0.035 &\textbf{1} &\textbf{1} &0.079 &0.008 &13.50 &0\% &0.375 &1.280 &1.059 &0.024 &\textbf{1} &\textbf{1} &\textbf{0.007} &0.006 & 6.543\\
    GAN &0.675 &2.560 &0.657 &0.048 &\textbf{1} &\textbf{1} &0.101 &0.011 &17.16 &0\% &1.200 &2.106 &1.203 &0.059 &\textbf{1} &\textbf{1} &0.076 &0.012 &10.70 \\
    Diffusion &0.050 &2.225 &0.781 &0.014 &\textbf{1} &\textbf{1} &0.048 &0.003 &11.53 &0\% &0.150 &1.717 &1.390 &0.031 &\textbf{1} &\textbf{1} &0.040 &0.004 &13.94 \\[0.2em]
    \cdashline{1-20}\\[-0.8em]
    HYGENE &0.775& \textbf{0.475}& \textbf{0.012}& \textbf{0.006}&\textbf{1} &\textbf{1} & \textbf{0.009}& \textbf{2.6e-4}& \textbf{2.127}& \textbf{65\%}& 0.525& \textbf{0.321}& \textbf{0.002}& \textbf{0.010}&\textbf{1} &\textbf{1} & 0.016& \textbf{4.4e-4}& \textbf{2.990} \\
    \bottomrule
    \end{tabular}%
    }
    \caption{Evaluation metrics for Erdos-Renyi and SBM hypergraphs}
    \label{tab:hypergraph_results}
\end{table*}

\vspace{0.5cm}

\begin{table*}[h!]  
    \centering
    \setlength{\tabcolsep}{3pt}  
    \resizebox{\textwidth}{!}{%
    \begin{tabular}{@{}l*{10}{c}*{10}{c}@{}}
    \toprule
    \multirow{3}{*}{\textbf{Model}} & \multicolumn{10}{c}{\textbf{Ego Hypergraphs} ($n_{avg} = $ 109.71, $std = $ 10.23)} & \multicolumn{9}{c}{\textbf{Tree hypergraphs} ($n_{avg} = $ 32, $std = $ 0)} \\
    \cmidrule(l){2-11}
    \cmidrule(l){12-21}
    &
    \makecell{\textbf{Valid} \\ \textbf{Ego} $\uparrow$} &
    \makecell{\textbf{Node} \\ \textbf{Num} $\downarrow$} &
    \makecell{\textbf{Node} \\ \textbf{Deg} $\downarrow$} &
    \makecell{\textbf{Edge} \\ \textbf{Size} $\downarrow$} &
    \makecell{\textbf{Spec-} \\ \textbf{tral} $\downarrow$} &
    \makecell{\textbf{Uniq.} $\uparrow$}&
    \makecell{\textbf{Nov.} $\uparrow$}&
    \makecell{\textbf{Cent.} \\ \textbf{Close} $\downarrow$} &
    \makecell{\textbf{Cent.} \\ \textbf{Betw.} $\downarrow$} &
    \makecell{\textbf{Cent.} \\ \textbf{Harm.} $\downarrow$} &
    \makecell{\textbf{Valid} \\ \textbf{Tree} $\uparrow$} &
    \makecell{\textbf{Node} \\ \textbf{Num} $\downarrow$} &
    \makecell{\textbf{Node} \\ \textbf{Deg} $\downarrow$} &
    \makecell{\textbf{Edge} \\ \textbf{Size} $\downarrow$} &
    \makecell{\textbf{Spec-} \\ \textbf{tral} $\downarrow$} &
    \makecell{\textbf{Uniq.} $\uparrow$}&
    \makecell{\textbf{Nov.} $\uparrow$}&
    \makecell{\textbf{Cent.} \\ \textbf{Close} $\downarrow$} &
    \makecell{\textbf{Cent.} \\ \textbf{Betw.} $\downarrow$} &
    \makecell{\textbf{Cent.} \\ \textbf{Harm.} $\downarrow$} \\
    \midrule
    HyperPA &0\% &35.83 &2.590 &0.423 &0.237 &\textbf{1} &\textbf{1} &0.354 &0.002 &143.0 &0\% &2.350 &0.315 &0.284 &0.159 &\textbf{1} &\textbf{1} &0.477 &0.168 &5.941\\
    VAE &0\% &47.58 &0.803 &1.458 &0.133 &\textbf{1} &\textbf{1} &0.558 &0.019 &38.95 &0\% &9.700 &0.072 &0.480 &0.124 &\textbf{1} &\textbf{1} &0.280 &0.139 &3.869 \\
    GAN &0\% &60.35 &0.917 &1.665 &0.230 &\textbf{1} &\textbf{1} &0.612 &0.015 &41.80 &0\% &6.0 &0.151 &0.469 &0.089 &\textbf{1} &\textbf{1} &0.201 &0.124 &2.198 \\
    Diffusion &0\% &\textbf{4.475} &3.984 &2.985 &0.190 &\textbf{1} &\textbf{1} &0.407 &0.009 &6.911 &0\% &2.225 &1.718 &1.922 &0.127 &\textbf{1} &\textbf{1} &0.353 &0.139 &8.565 \\[0.2em]
    \cdashline{1-21}\\[-0.8em]
    Ours & \textbf{90\%}& 12.55& \textbf{0.063}& \textbf{0.220}& \textbf{0.004}&\textbf{1} &\textbf{1} & \textbf{0.025}& \textbf{8.95e-5}& \textbf{5.790}& \textbf{77.5\%}& \textbf{0}& \textbf{0.059}& \textbf{0.108}& \textbf{0.012}&\textbf{1} &\textbf{1} & \textbf{0.041}& \textbf{0.016}& \textbf{1.099}\\
    \bottomrule
    \end{tabular}%
    }
    \caption{Evaluation metrics for Ego and Tree hypergraphs}
    \label{tab:hypergraph_results_ego_tree} 
\end{table*}

\vspace{0.5cm}

\begin{table*}[h!] 
    \centering
    \setlength{\tabcolsep}{3pt}  
    \resizebox{\textwidth}{!}{%
    \begin{tabular}{@{}l*{9}{c}*{9}{c}@{}}
    \toprule
    \multirow{3}{*}{\textbf{Model}} & \multicolumn{9}{c}{\textbf{ModelNet40 Plant} ($n_{avg} = $ 124.86, $std = $ 87.88)} & \multicolumn{9}{c}{\textbf{ModelNet40 Bookshelf} ($n_{avg} = $ 119.38, $std = $ 68.20)} \\
    \cmidrule(l){2-10}
    \cmidrule(l){11-19}
    &
    \makecell{\textbf{Node} \\ \textbf{Num} $\downarrow$} &
    \makecell{\textbf{Node} \\ \textbf{Deg} $\downarrow$} &
    \makecell{\textbf{Edge} \\ \textbf{Size} $\downarrow$} &
    \makecell{\textbf{Spec-} \\ \textbf{tral} $\downarrow$} &
    \makecell{\textbf{Uniq.} $\uparrow$}&
    \makecell{\textbf{Nov.} $\uparrow$}&
    \makecell{\textbf{Cent.} \\ \textbf{Close} $\downarrow$} &
    \makecell{\textbf{Cent.} \\ \textbf{Betw.} $\downarrow$} &
    \makecell{\textbf{Cent.} \\ \textbf{Harm.} $\downarrow$} &
    \makecell{\textbf{Node} \\ \textbf{Num} $\downarrow$} &
    \makecell{\textbf{Node} \\ \textbf{Deg} $\downarrow$} &
    \makecell{\textbf{Edge} \\ \textbf{Size} $\downarrow$} &
    \makecell{\textbf{Spec-} \\ \textbf{tral} $\downarrow$} &
    \makecell{\textbf{Uniq.} $\uparrow$}&
    \makecell{\textbf{Nov.} $\uparrow$}&
    \makecell{\textbf{Cent.} \\ \textbf{Close} $\downarrow$} &
    \makecell{\textbf{Cent.} \\ \textbf{Betw.} $\downarrow$} &
    \makecell{\textbf{Cent.} \\ \textbf{Harm.} $\downarrow$} \\
    \midrule
    HyperPA &10.82 &6.566 &0.046 &0.061 &\textbf{1} &\textbf{1} &0.266 &0.009 &932.9 &8.025 &7.562 &0.044 &\textbf{0.048} &\textbf{1} &\textbf{1} &0.211 &0.005 &877.5 \\
    VAE &76.15 &3.895 &1.573 &0.205 &\textbf{1} &\textbf{1} &\textbf{0.230} &\textbf{0.005} &\textbf{73.50} &47.45 &6.190 &1.520 &0.190 &\textbf{1} &\textbf{1} &\textbf{0.145} &\textbf{0.003} &113.6\\
    GAN &\textbf{0} &378.1 &56.35 &0.364 &\textbf{1} &\textbf{1} &0.782 &0.012 &644.8 &\textbf{0.000} &397.2 &46.30 &0.476 &\textbf{1} &\textbf{1} &0.707 &0.007 &670.1 \\
    Diffusion &0.025 &21.03 &3.439 &0.069 &\textbf{1} &\textbf{1} &0.319 &0.010 &270.2 &\textbf{0.000} &20.36 &2.346 &0.079 &\textbf{1} &\textbf{1} &0.239 &0.006 &264.1 \\[0.2em]
    \cdashline{1-19}\\[-0.8em]
    Ours &68.38 &\textbf{2.428} &\textbf{0.027} &\textbf{0.034} &\textbf{1} &\textbf{1} &0.263 &0.009 &197.1 &69.73 &\textbf{1.050} &\textbf{0.034} &0.068 &\textbf{1} &\textbf{1} &0.204 &0.004 &\textbf{27.40} \\
    \bottomrule
    \end{tabular}%
    }
    \caption{Evaluation metrics for ModelNet40 Plant and ModelNet40 Bookshelf}
    \label{tab:modelnet40_results}
\end{table*}

\vspace{0.5cm}

\begin{table*}[h!] 
    \centering
    \setlength{\tabcolsep}{3pt}  
    \resizebox{0.5\textwidth}{!}{%
    \begin{tabular}{@{}l*{9}{c}@{}}
    \toprule
    \multirow{3}{*}{\textbf{Model}} & \multicolumn{9}{c}{\textbf{ModelNet40 Piano} ($n_{avg} = $ 177.29, $std = $ 57.11)} \\
    \cmidrule(l){2-10}
    &
    \makecell{\textbf{Node} \\ \textbf{Num} $\downarrow$} &
    \makecell{\textbf{Node} \\ \textbf{Deg} $\downarrow$} &
    \makecell{\textbf{Edge} \\ \textbf{Size} $\downarrow$} &
    \makecell{\textbf{Spec-} \\ \textbf{tral} $\downarrow$} &
    \makecell{\textbf{Uniq.} $\uparrow$}&
    \makecell{\textbf{Nov.} $\uparrow$}&
    \makecell{\textbf{Cent.} \\ \textbf{Close} $\downarrow$} &
    \makecell{\textbf{Cent.} \\ \textbf{Betw.} $\downarrow$} &
    \makecell{\textbf{Cent.} \\ \textbf{Harm.} $\downarrow$} \\
    \midrule
    HyperPA &0.825 &9.254 &\textbf{0.023} &\textbf{0.067} &\textbf{1} &\textbf{1} &\textbf{0.236} &0.004 &\textbf{77.84} \\
    VAE &75.35 &8.060 &1.686 &0.396 &\textbf{1} &\textbf{1} &0.241 &0.003 &184.3 \\
    GAN &\textbf{0.000} &409.0 &86.38 &0.697 &\textbf{1} &\textbf{1} &0.738 &0.005 &622.2 \\
    Diffusion &0.050 &20.90 &4.192 &0.113 &\textbf{1} &\textbf{1} &0.303 &0.004 &289.3 \\[0.2em]
    \cdashline{1-10}\\[-0.8em]
    Ours &42.52 &\textbf{6.290} &0.027 &0.117 &\textbf{1} &\textbf{1} &0.285 &\textbf{0.002} &155.0 \\
    \bottomrule
    \end{tabular}%
    }

    \caption{Evaluation metrics for ModelNet40 Piano}
    \label{tab:hypergraph_results_piano}
\end{table*}

\newpage
\section{Comparison Between Training and Generated Samples} \label{app:visual_examples}

\begin{figure}[h]
    \centering
    \begin{subfigure}[b]{0.47\textwidth}
        \centering
        \begin{subfigure}[b]{0.47\textwidth}
            \centering
            \includegraphics[width=\textwidth]{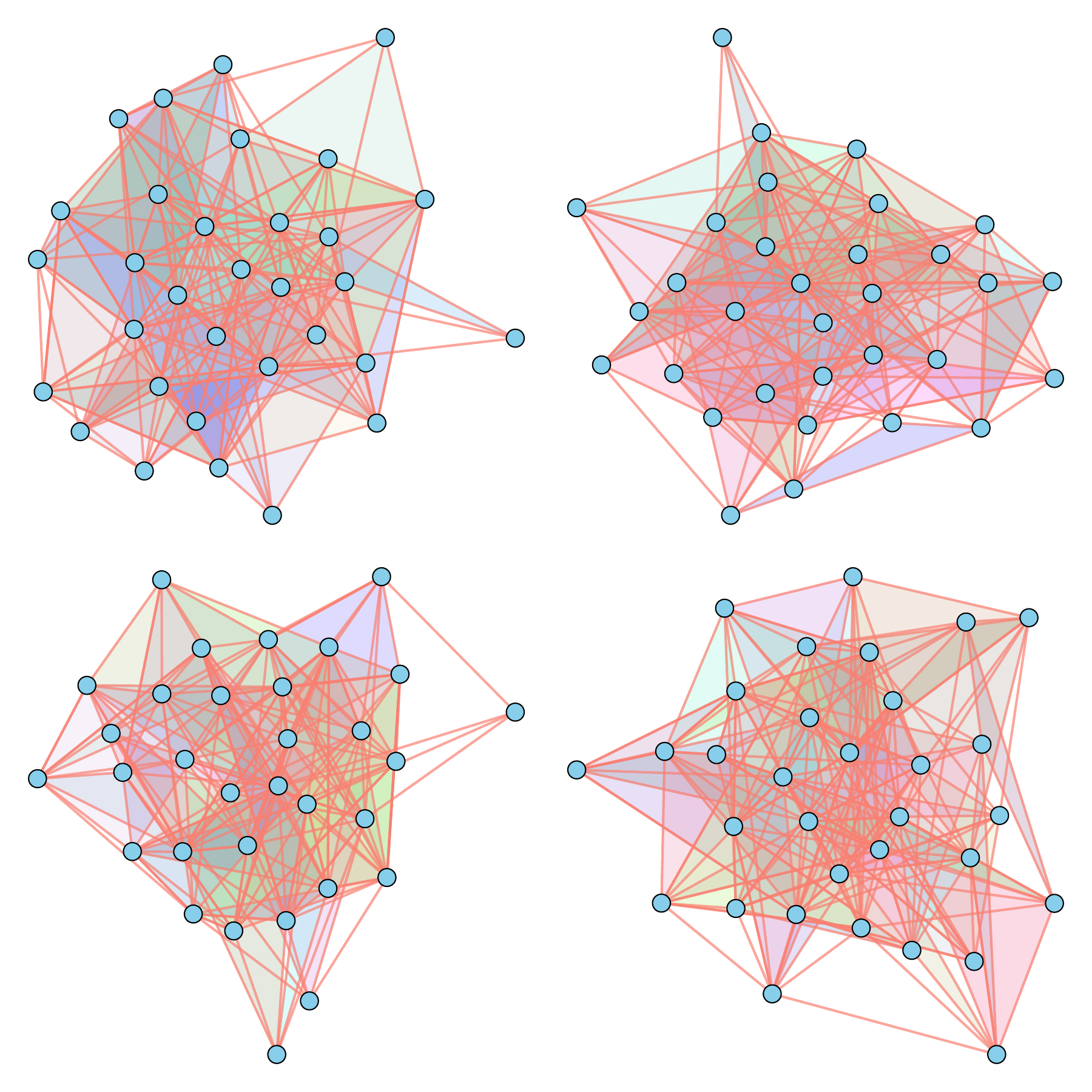}
            \caption*{Train samples}
        \end{subfigure}
        \hfill
        \begin{subfigure}[b]{0.47\textwidth}
            \centering
            \includegraphics[width=\textwidth]{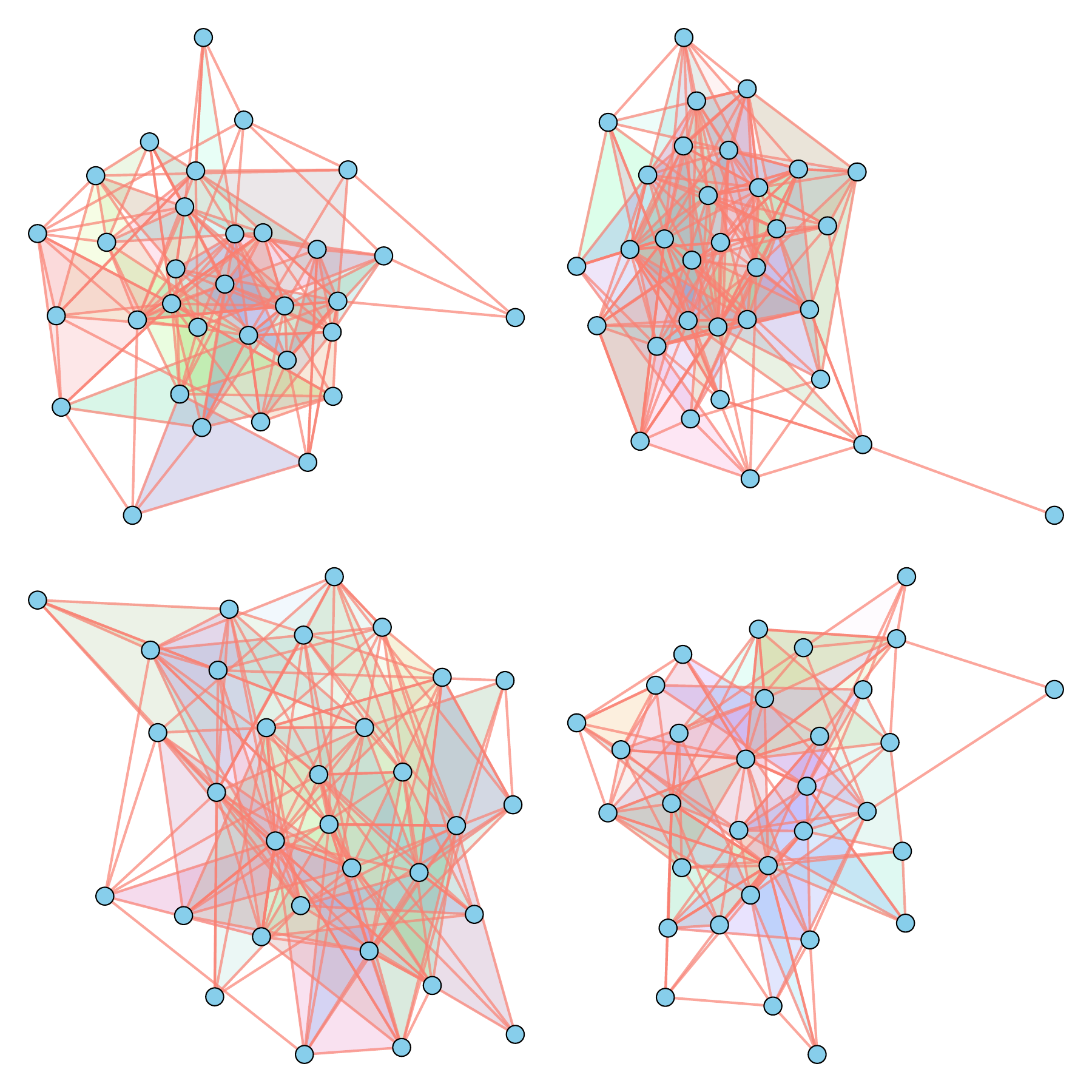}
            \caption*{Generated samples}
        \end{subfigure}
        \vspace{-0.15cm}
        \caption{Erdos-Renyi hypergraphs}
    \end{subfigure}
    \hfill
    \begin{subfigure}[b]{0.47\textwidth}
        \centering
        \begin{subfigure}[b]{0.47\textwidth}
            \centering
            \includegraphics[width=\textwidth]{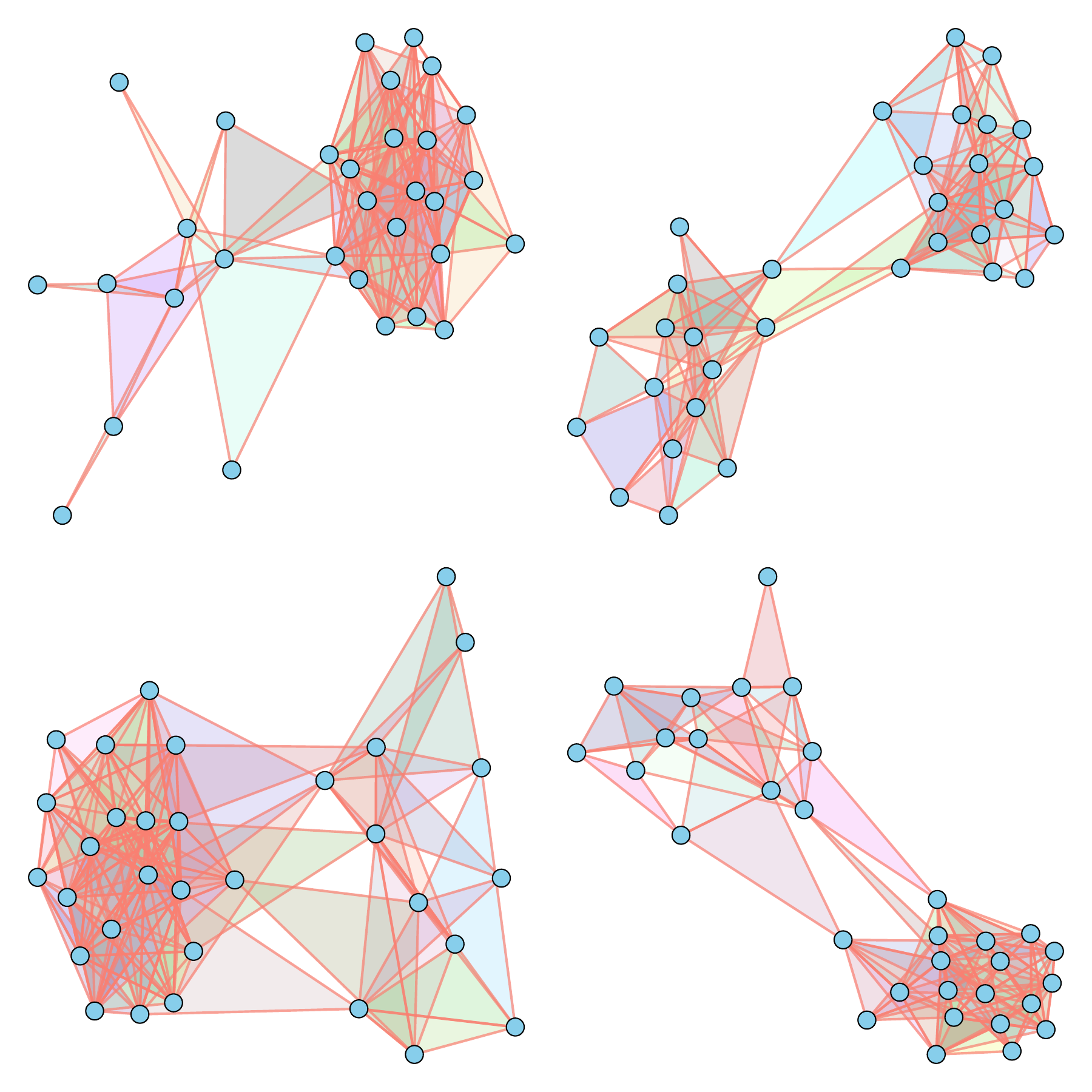}
            \caption*{Train samples}
        \end{subfigure}
        \hfill
        \begin{subfigure}[b]{0.47\textwidth}
            \centering
            \includegraphics[width=\textwidth]{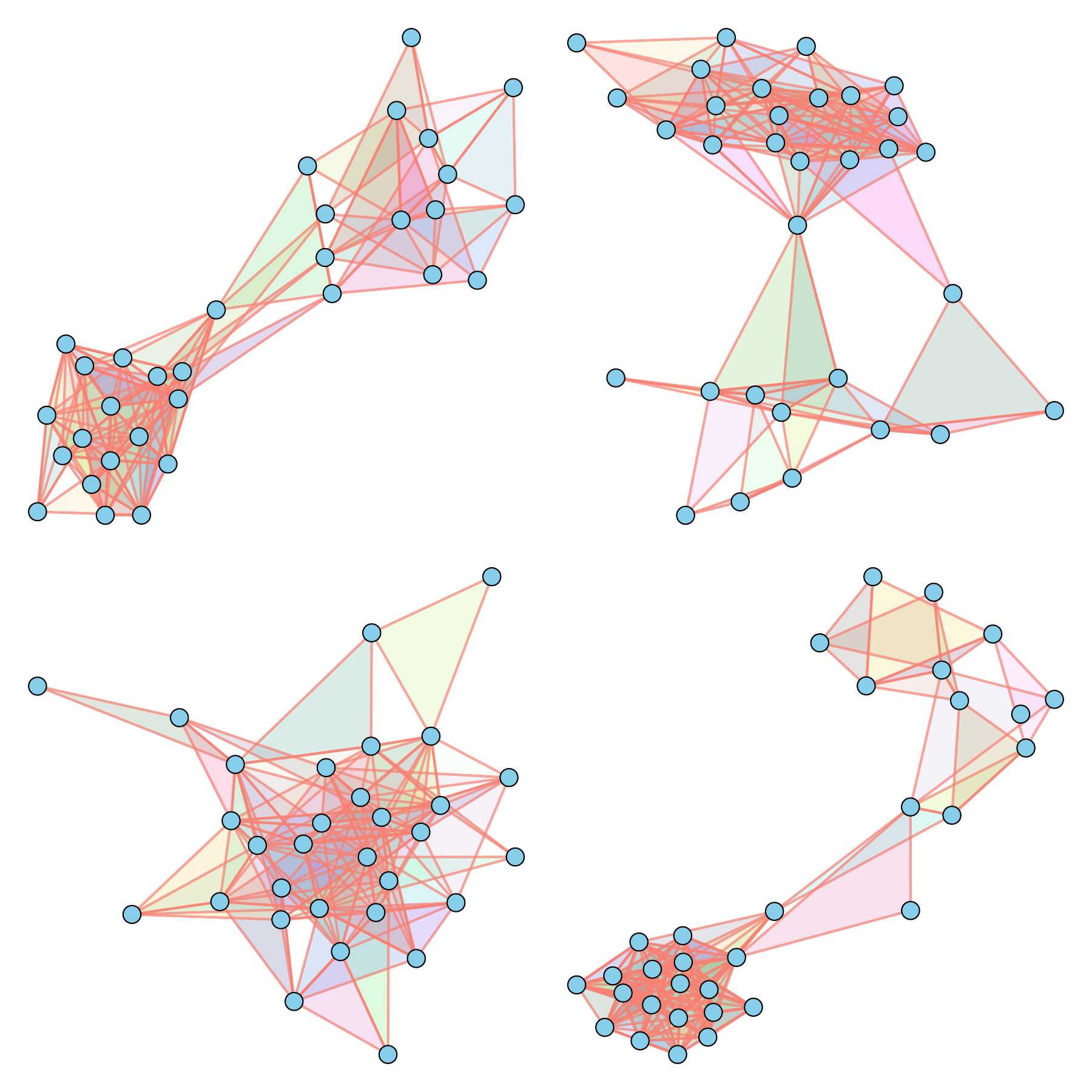}
            \caption*{Generated samples}
        \end{subfigure}
        \vspace{-0.15cm}
        \caption{Stochastic Block Model hypergraphs}
    \end{subfigure}

     \vspace{0.31cm}

    \begin{subfigure}[b]{0.47\textwidth}
        \centering
        \begin{subfigure}[b]{0.47\textwidth}
            \centering
            \includegraphics[width=\textwidth]{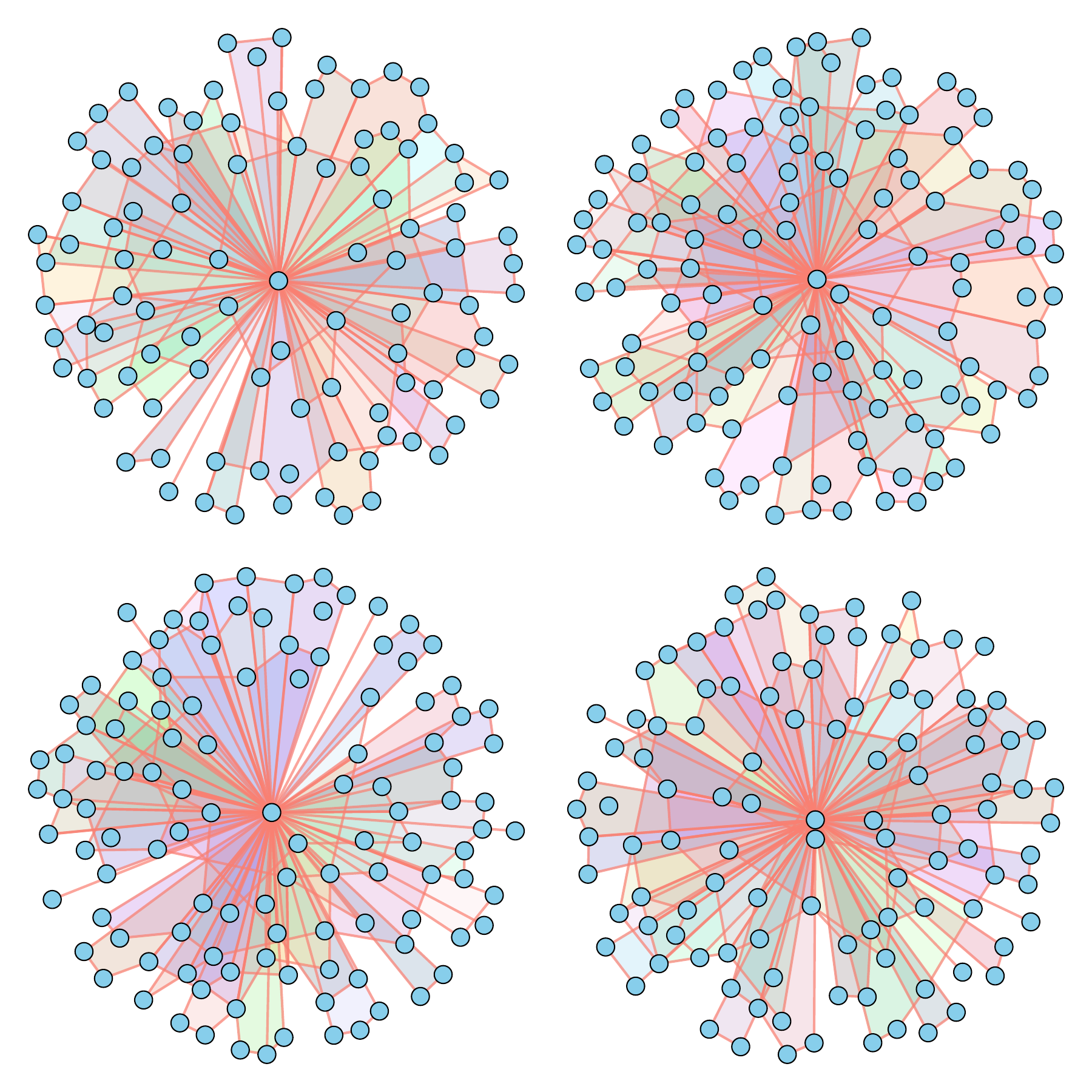}
            \caption*{Train samples}
        \end{subfigure}
        \hfill
        \begin{subfigure}[b]{0.47\textwidth}
            \centering
            \includegraphics[width=\textwidth]{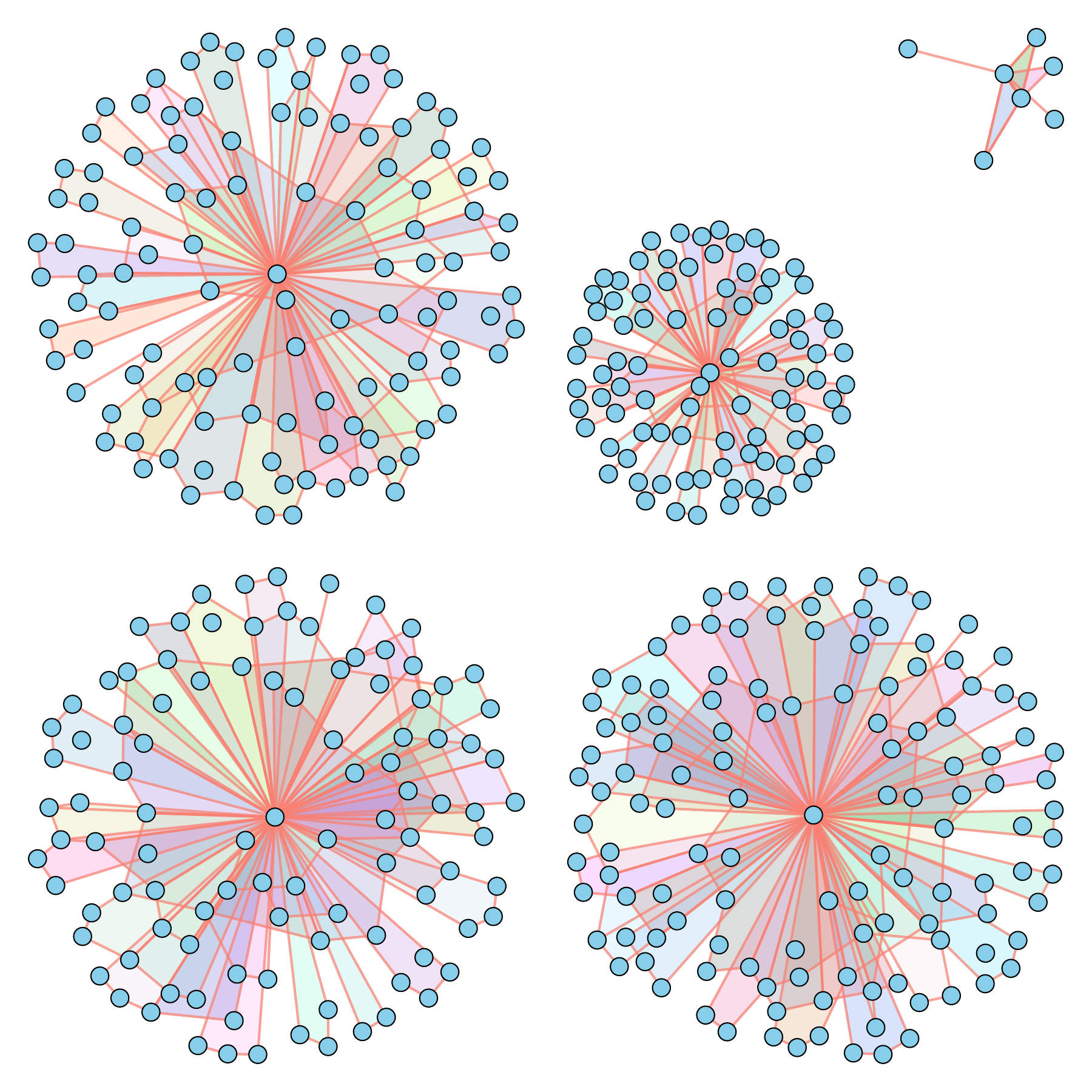}
            \caption*{Generated samples}
        \end{subfigure}
        \vspace{-0.15cm}
        \caption{Ego hypergraphs}
    \end{subfigure}
    \hfill
    \begin{subfigure}[b]{0.47\textwidth}
        \centering
        \begin{subfigure}[b]{0.47\textwidth}
            \centering
            \includegraphics[width=\textwidth]{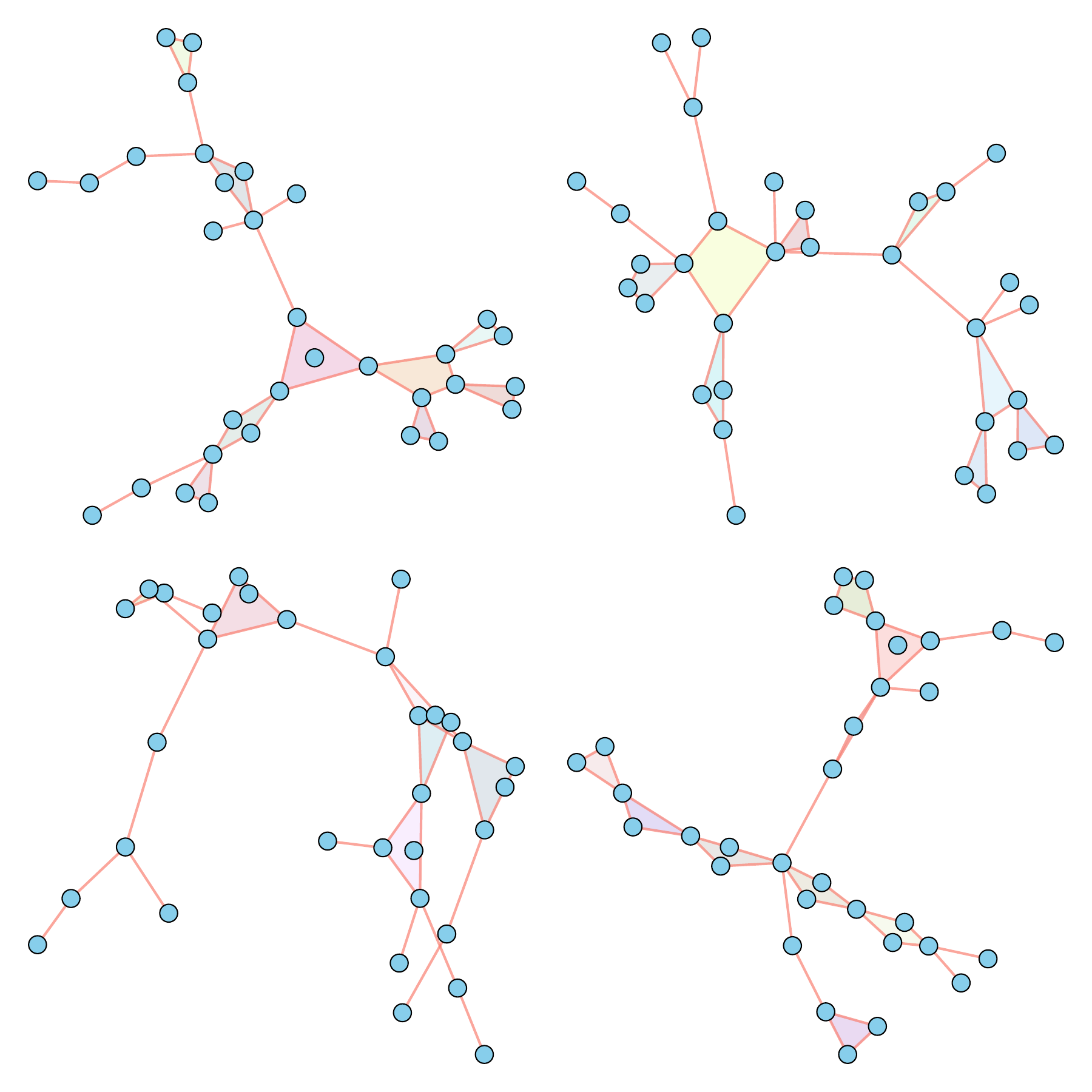}
            \caption*{Train samples}
        \end{subfigure}
        \hfill
        \begin{subfigure}[b]{0.47\textwidth}
            \centering
            \includegraphics[width=\textwidth]{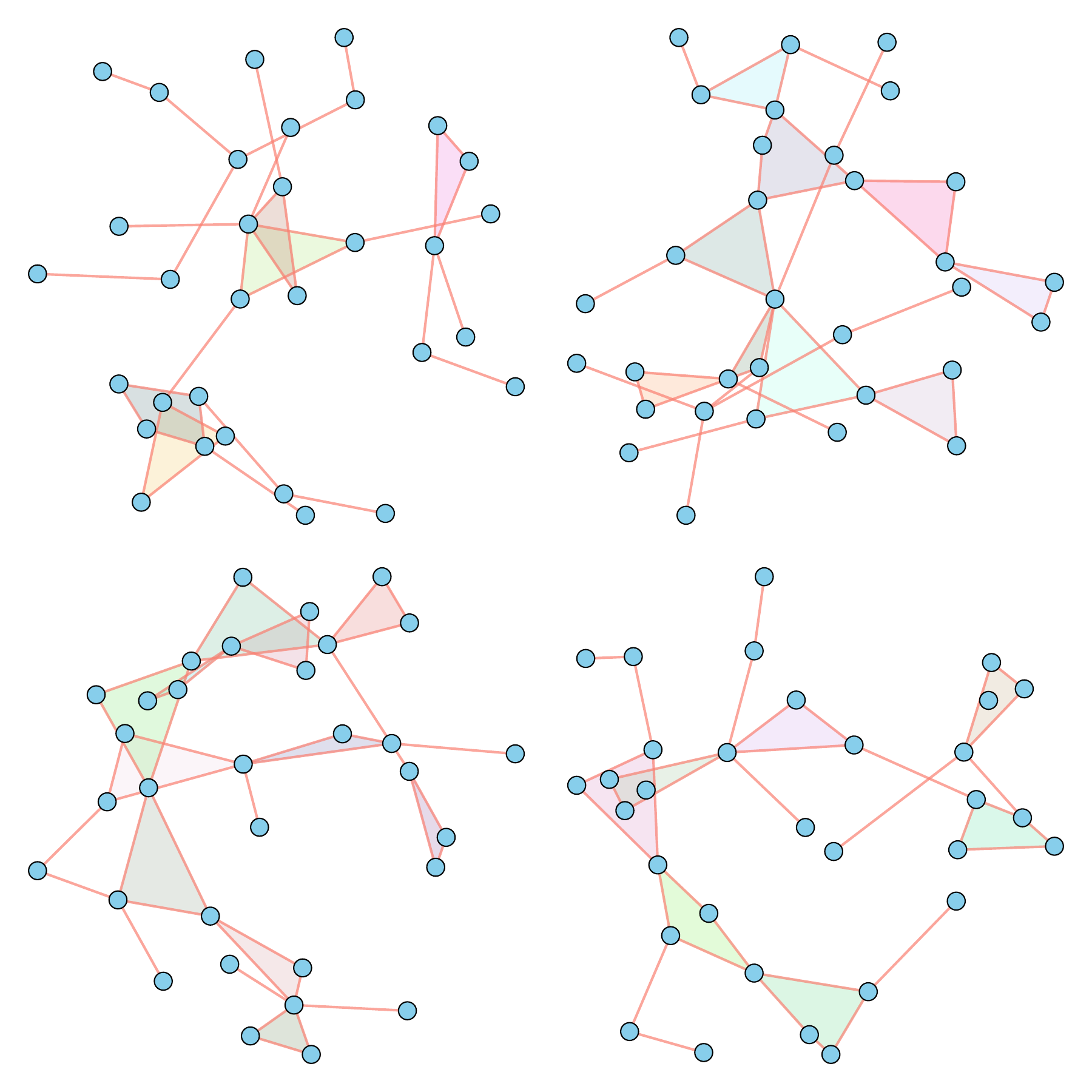}
            \caption*{Generated samples}
        \end{subfigure}
        \vspace{-0.15cm}
        \caption{Tree hypergraphs}
    \end{subfigure}

    \vspace{0.31cm}

    \begin{subfigure}[b]{0.47\textwidth}
        \centering
        \begin{subfigure}[b]{0.47\textwidth}
            \centering
            \includegraphics[width=\textwidth]{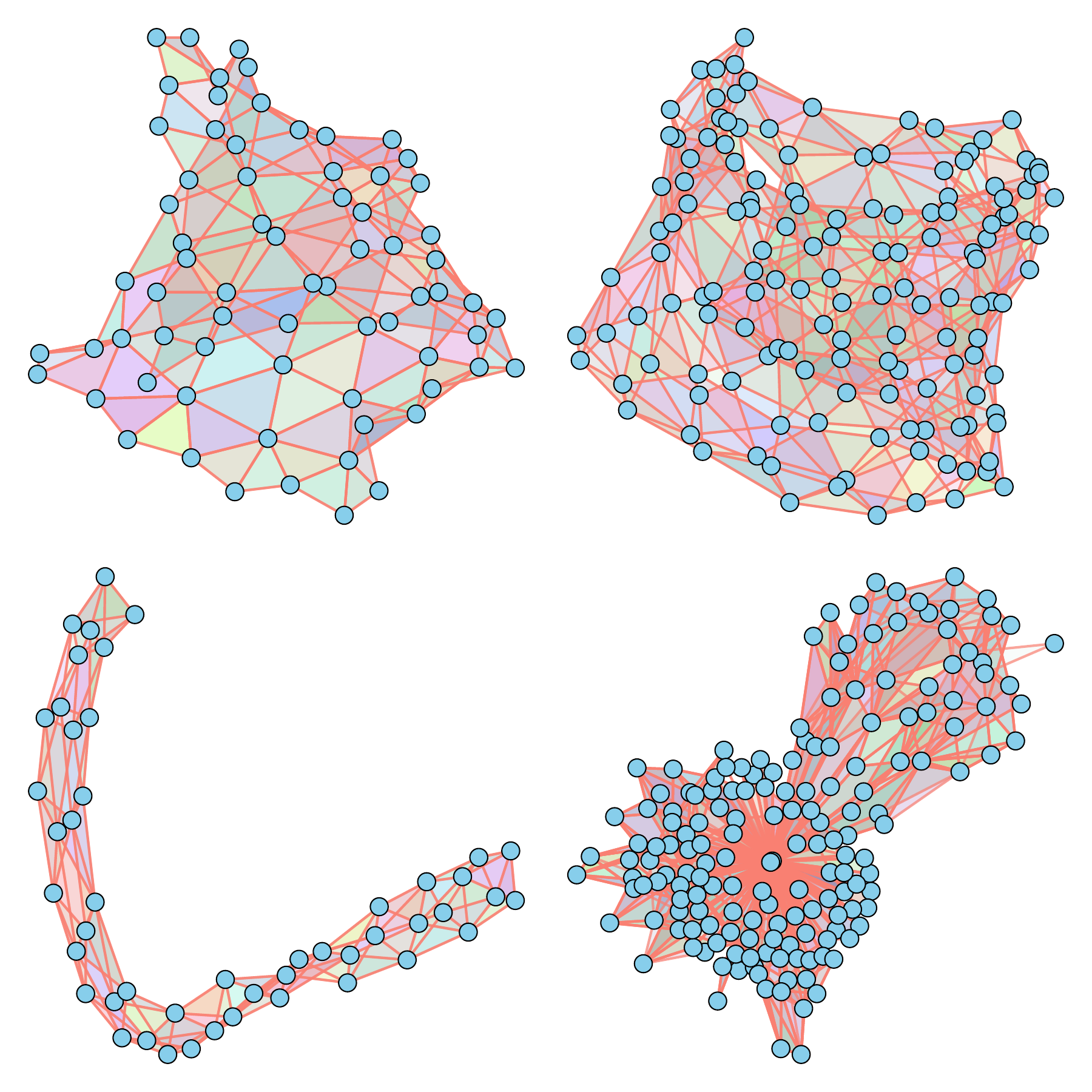}
            \caption*{Train samples}
        \end{subfigure}
        \hfill
        \begin{subfigure}[b]{0.47\textwidth}
            \centering
            \includegraphics[width=\textwidth]{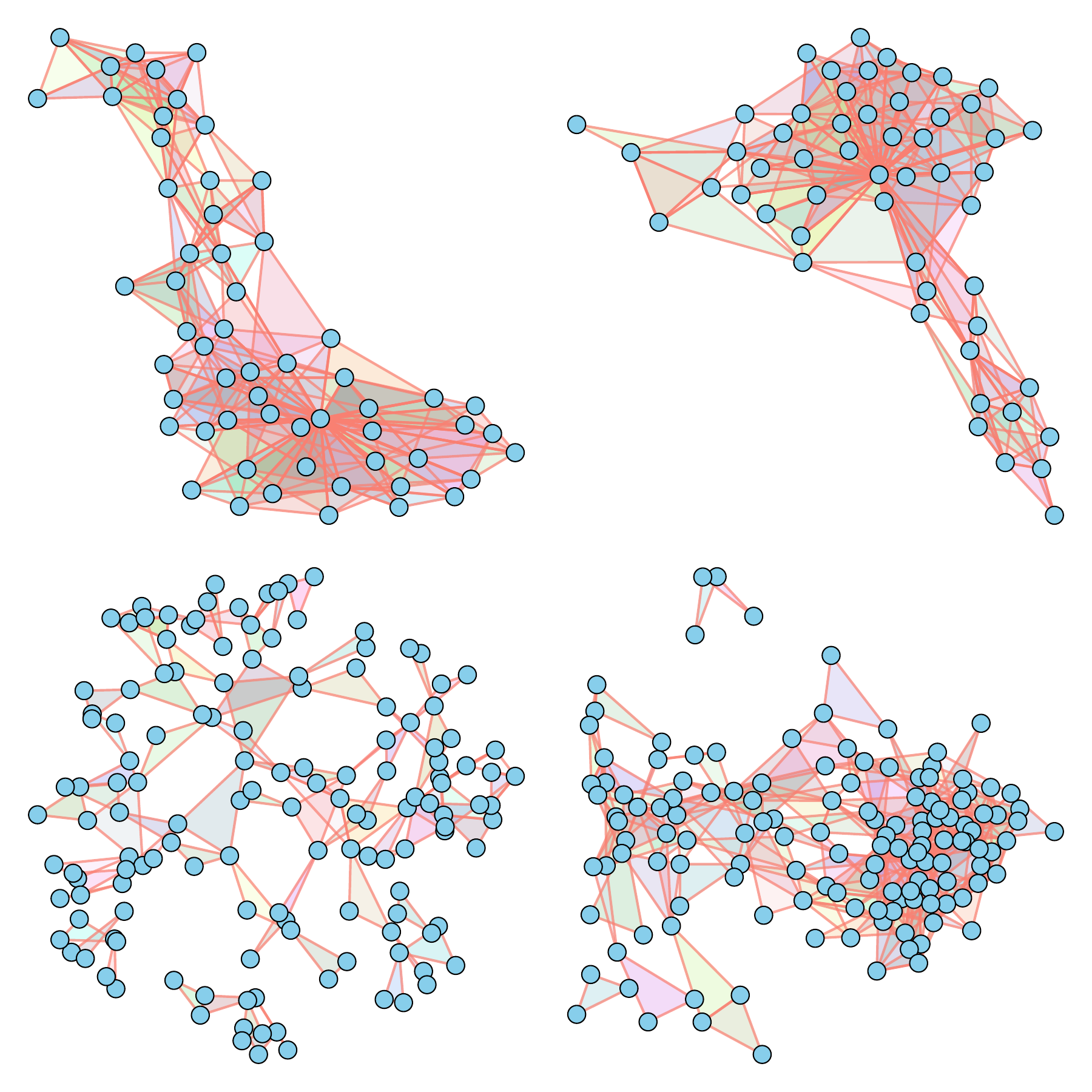}
            \caption*{Generated samples}
        \end{subfigure}
        \vspace{-0.15cm}
        \caption{Plant meshes topology}
    \end{subfigure}
    \hfill
    \begin{subfigure}[b]{0.47\textwidth}
        \centering
        \begin{subfigure}[b]{0.47\textwidth}
            \centering
            \includegraphics[width=\textwidth]{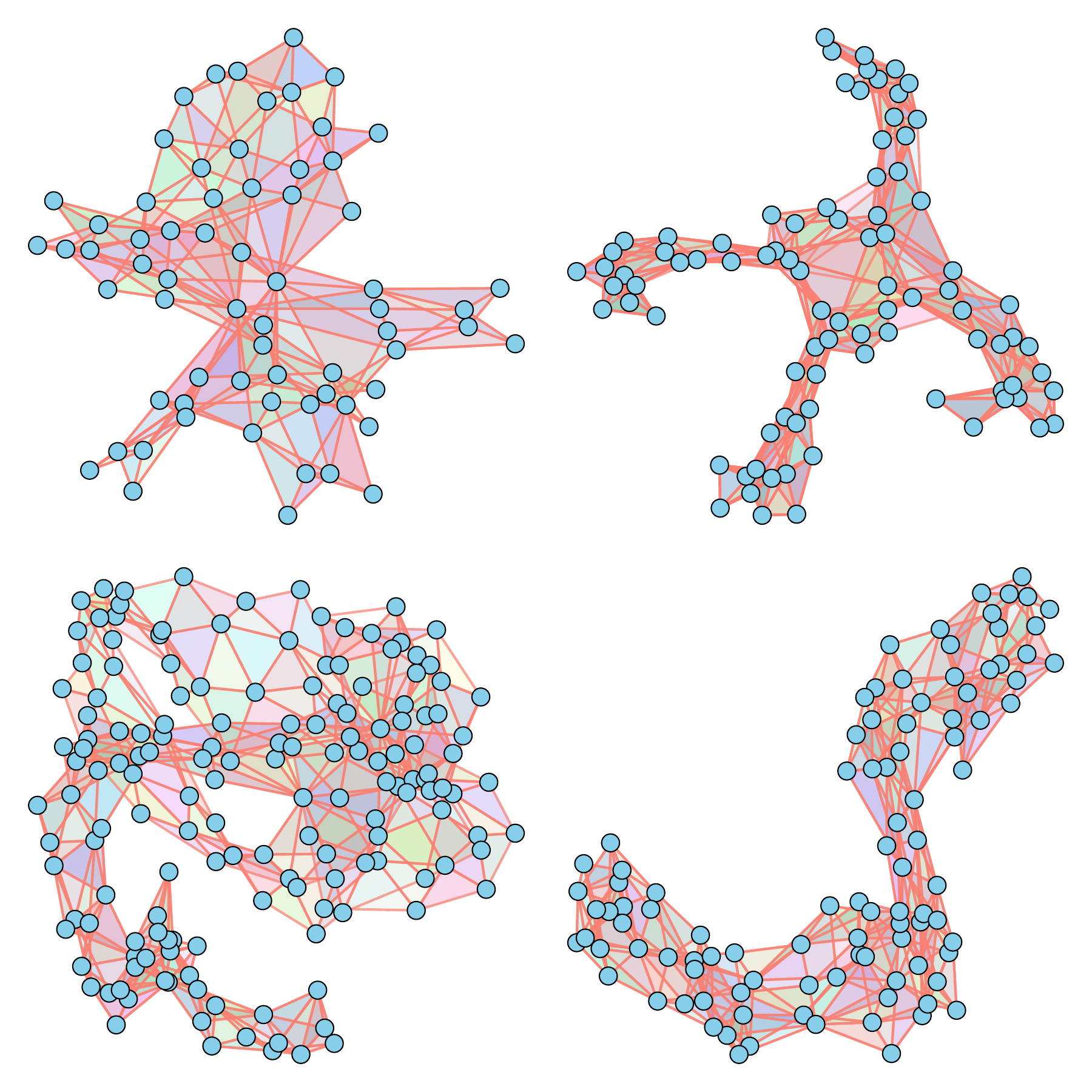}
            \caption*{Train samples}
        \end{subfigure}
        \hfill
        \begin{subfigure}[b]{0.47\textwidth}
            \centering
            \includegraphics[width=\textwidth]{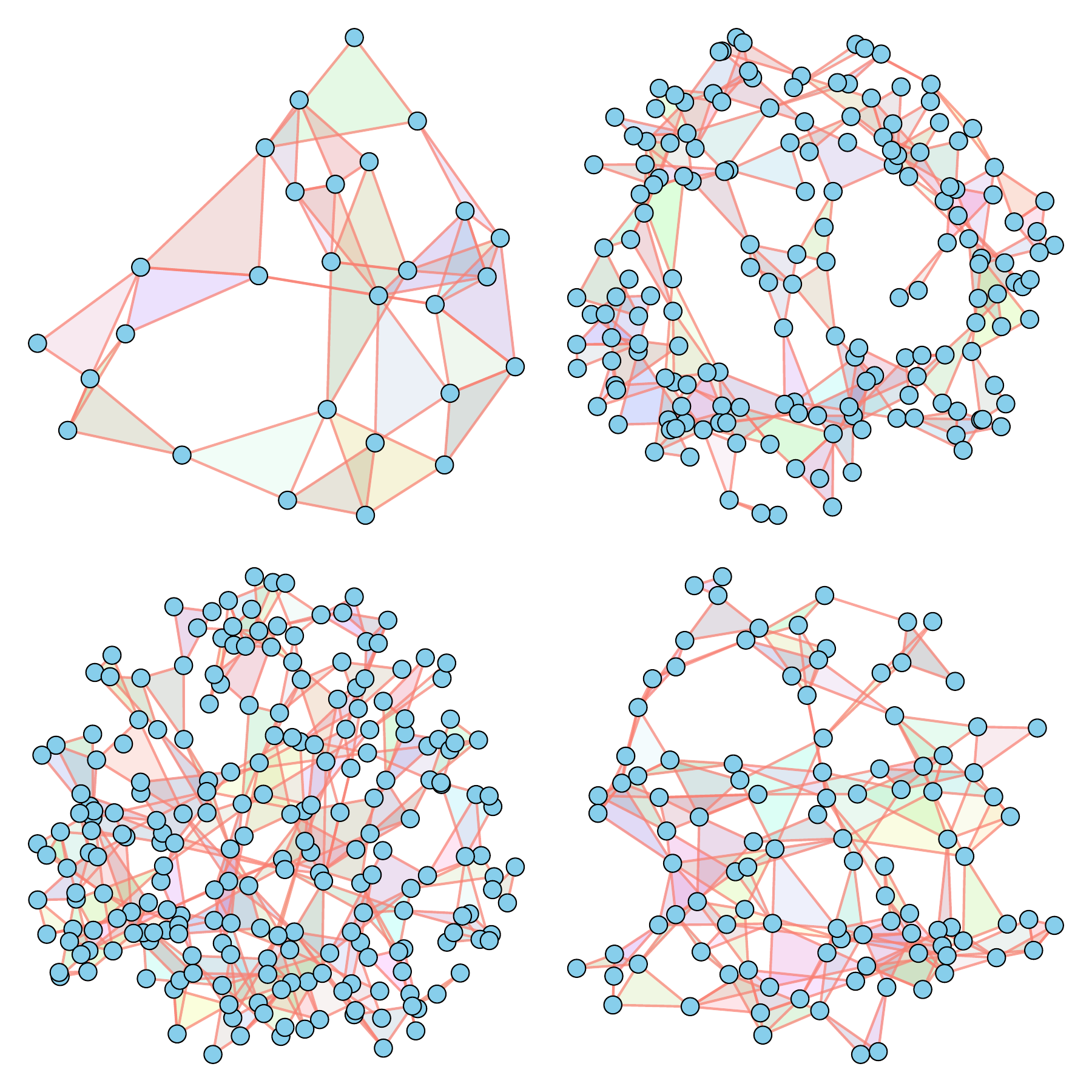}
            \caption*{Generated samples}
        \end{subfigure}
        \vspace{-0.15cm}
        \caption{Bookshelf meshes topology}
    \end{subfigure}

    \vspace{0.31cm}

    \begin{subfigure}[b]{0.47\textwidth}
        \centering
        \begin{subfigure}[b]{0.47\textwidth}
            \centering
            \includegraphics[width=\textwidth]{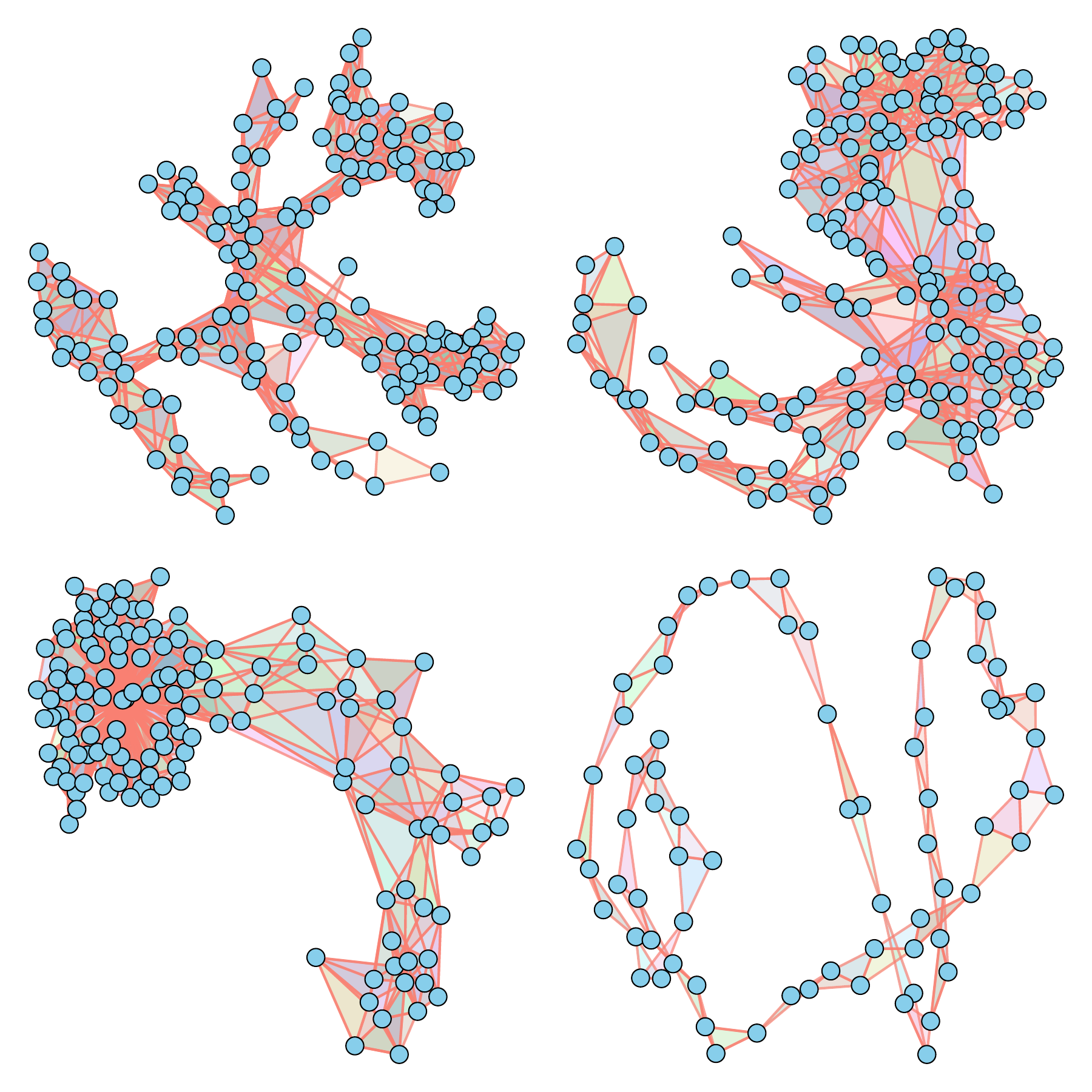}
            \caption*{Train samples}
        \end{subfigure}
        \hfill
        \begin{subfigure}[b]{0.47\textwidth}
            \centering
            \includegraphics[width=\textwidth]{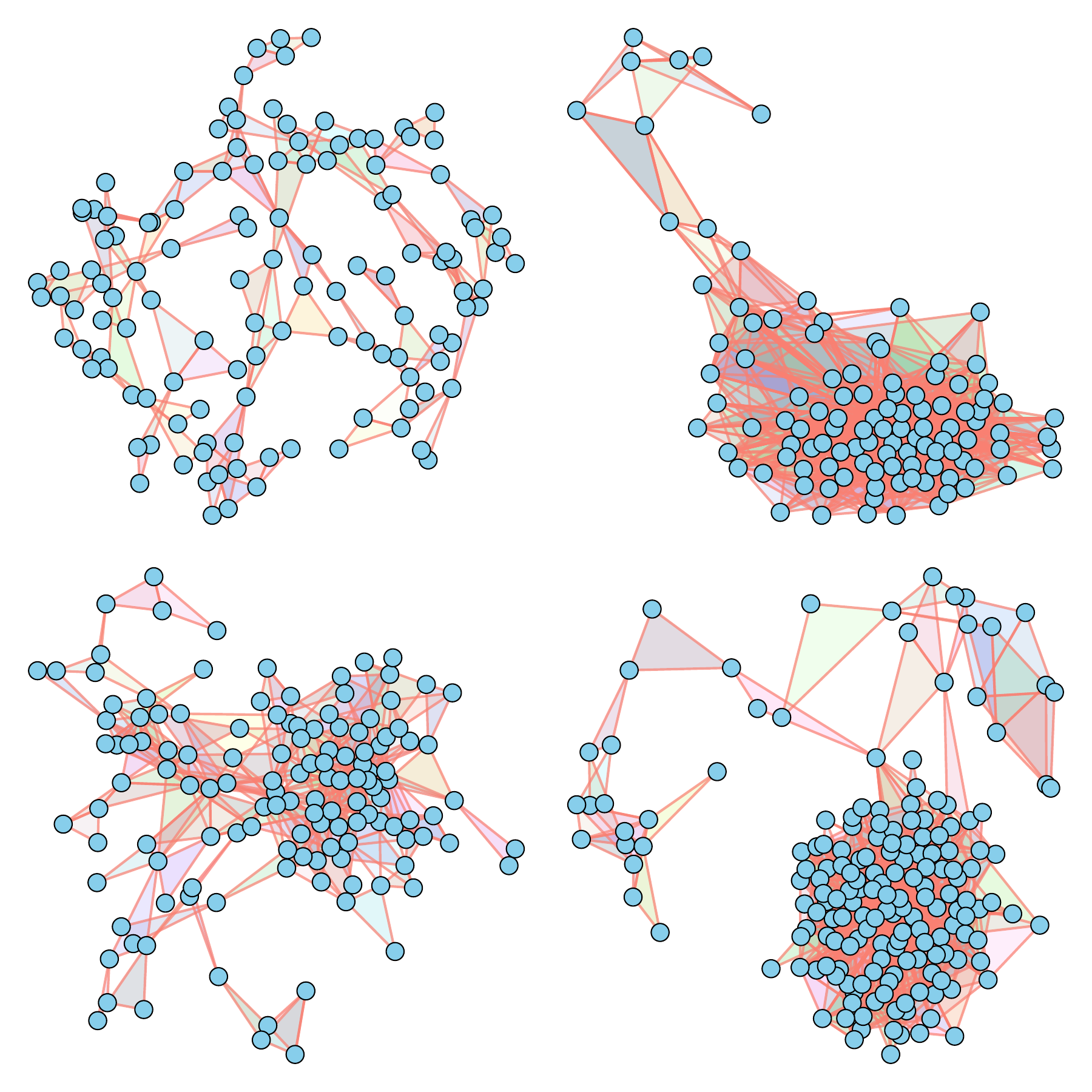}
            \caption*{Generated samples}
        \end{subfigure}
        \vspace{-0.15cm}
        \caption{Piano meshes topology}
    \end{subfigure}

    \caption{Comparison of train and generated samples for various datasets}
    \vspace{-20em}
\end{figure}

\end{document}